\documentclass[journal]{IEEEtran}

\usepackage{graphicx}
\usepackage{subfigure}
\usepackage{cite}
\usepackage[usenames]{color}
\usepackage{bm}
\usepackage{amsmath}
\usepackage{amssymb}
\usepackage[vlined,linesnumbered,boxruled]{algorithm2e}

\usepackage{flushend}

\usepackage{theorem}
\newtheorem{theorem}{Theorem}
\newtheorem{lemma}[theorem]{Lemma}
\newtheorem{proposition}[theorem]{Proposition}
\newtheorem{corollary}[theorem]{Corollary}

\newtheorem{problem}[theorem]{Problem}

\newcommand{\reals}{\ensuremath{\mathbb{R}}}
\newcommand{\naturals}{\ensuremath{\mathbb{N}}}
\newcommand{\integers}{\ensuremath{\mathbb{Z}}}
\newcommand{\plane}{\ensuremath{\reals^2}}       
\newcommand{\lattice}{\ensuremath{\integers^2}}
\newcommand{\ball}[2]{\ensuremath{{B_{#2}(#1)}}} 
\newcommand{\BoxRegion}[2]{\ensuremath{{B(l_1,l_2)}}}
\newcommand{\countables}[1]{\ensuremath{{\cal C}_{#1}}} 

\newcommand{\PP}{\ensuremath{\mathbb{P}}}
\newcommand{\Event}{{E}} 

\newcommand{\BorelSets}[1]{{{\cal B}(#1)}}   

\newcommand{\PointProcess}{P}

\newcommand{\PoissonProcess}[1]{{\Pi_{#1}}}

\newcommand{\BorelMeasure}{{\xi}}         
\newcommand{\CountingMeasure}{{n}}   
\newcommand{\SpaceCountingMeasure}{{{\cal N}}} 
\newcommand{\EventCountingMeasure}{{N}} 

\newcommand{\MarkSpace}{{\cal K}}
\newcommand{\MarkMeasure}{K}
\newcommand{\MarkedCountingMeasure}{{m}}
\newcommand{\SpaceMarkedCountingMeasure}{{{\cal M}_{\MarkSpace}}}
\newcommand{\EventMarkedCountingMeasure}{{M}}
\newcommand{\MarkedPointProcess}{{Q}}
\newcommand{\MarkedPointProcessMeasure}[1]{{{\cal P}_{#1}}}

\newcommand{\MarkedInvariantSigmaAlgebra}{{\cal I}}

\newcommand{\GroundPointProcess}[1]{{{#1}_g}}

\newcommand{\Translate}[2]{{T_{#1}{#2}}}
\newcommand{\TranslateMeasure}[2]{{S_{#1}{#2}}}
\newcommand{\InverseTranslateMeasure}[2]{{S^{-1}_{#1}{#2}}}

\newcommand{\TranslateVector}{{v}}

\newcommand{\ElementPlane}{{y}}
\newcommand{\ElementMark}{{k}}
\newcommand{\Radius}{{r}}

\newcommand{\ElementLattice}{{y}}

\newcommand{\SubsetPlane}{{A}}
\newcommand{\SubsetMark}{{K}}
\newcommand{\SubsetPlaneMark}{{B}}

\newcommand{\Probability}{\ensuremath{p}}
\newcommand{\CriticalProbability}{\ensuremath{p_\mathrm{crit}^\mathrm{b}}}

\newcommand{\OpenCluster}[1]{\ensuremath{W_{#1}^0}}
\newcommand{\InfiniteOpenCluster}[1]{\ensuremath{W_{#1}}}

\newcommand{\OpenClusterSite}[1]{\ensuremath{V_{#1}^0}}
\newcommand{\CriticalProbabilitySite}{\ensuremath{p_\mathrm{crit}^\mathrm{s}}}

\newcommand{\OccupiedRegion}[2]{W_{#1, #2}} 
\newcommand{\OccupiedComponent}[3]{W_{#1,#2}^{#3}}
\newcommand{\VacantRegion}[2]{V_{#1, #2}} 
\newcommand{\VacantComponent}[3]{V_{#1,#2}^{#3}}
\newcommand{\criticaldegree}{d_c}  

\newcommand{\w}{\ensuremath{w}} 
\newcommand{\wmax}{\ensuremath{w_\mathrm{max}}} 
\newcommand{\speedratio}{\ensuremath{\nu}} 
\newcommand{\speed}{\ensuremath{\nu}}

\newcommand{\Maneuverability}{\ensuremath{\alpha}}

\newcommand{\anglespeed}{\ensuremath{\alpha}}

\newcommand{\U}{\ensuremath{{\cal U}}} 
\newcommand{\X}{\ensuremath{{\cal X}}} 

\newcommand{\SetTrajectories}[1]{\ensuremath{\ensuremath{\Gamma_{#1}}}}
\newcommand{\Trajectory}{\ensuremath{y}}

\newcommand{\Reparametrization}{\ensuremath{\sigma}}

\newcommand{\inputvar}[1]{\ensuremath{u_{#1}}} 
\newcommand{\Inputvar}[1]{\ensuremath{U_{#1}}}

\newcommand{\Statevar}[2]{\ensuremath{X_{#1,#2}}}
\newcommand{\Outputvar}[2]{\ensuremath{Y_{#1,#2}}}

\newcommand{\ReachableStates}[3]{{{\cal R}_{#1}(#2,#3)}}

\newcommand{\LeftPrimaryShadow}[3]{{\ensuremath{V_{#1,#3}(#2)}}}
\newcommand{\RightPrimaryShadow}[3]{{\ensuremath{W_{#1,#3}(#2)}}}
\newcommand{\LeftShadow}[3]{\ensuremath{{\cal V}_{#1,#3}(#2)}}

\newcommand{\PrimaryShadow}[3]{{\ensuremath{S_{#1,#3}(#2)}}}

\newcommand{\ShadowDummy}{\ensuremath{S}}
\newcommand{\InducedShadow}[2]{{\ensuremath{{\tt Ind}(#1,#2)}}}

\newcommand{\TimeInterval}{\ensuremath{I}}
\newcommand{\Time}{\ensuremath{T}}

\newcommand{\InputSequence}{\ensuremath{\sigma}}

\newcommand{\OutputTrajectories}[1]{\ensuremath{\Psi_{#1}}}
\newcommand{\OutputSwitchPoints}[1]{\ensuremath{\psi}_{#1}}

\newcommand{\poissonforest}[1]{\ensuremath{\Pi_{#1}}}
\newcommand{\RandomTreeRadius}{{R}}   
\newcommand{\randomTreeRadius}[1]{{r_{#1}}}

\newcommand{\ForestRealizations}{\Phi}

\newcommand{\ForestGeneratingProcess}{F}

\newcommand{\Width}{w}
\newcommand{\Length}{l}
\newcommand{\Region}[2]{\ensuremath{{\cal R}(l,w)}}

\newcommand{\planner}{\ensuremath{p}} 
\newcommand{\plannerControl}{\ensuremath{g}} 
\newcommand{\plannerInitState}{\ensuremath{h}} 

\newcommand{\SetTreeLocations}{\ensuremath{{\cal Y}}}

\newcommand{\FreeGeneral}[1]{{{\cal X}_\mathrm{free}^{#1}}}  
\newcommand{\TreesGeneral}[1]{{{\cal X}_\mathrm{trees}^{#1}}}

\newcommand{\Card}[1]{\ensuremath{\vert #1 \vert}}

\newcommand{\LOne}[1]{\ensuremath{\Vert #1 \Vert_1}}

\begin{document}

\title{High-speed Flight in an Ergodic Forest}

\author{Sertac~Karaman \and \qquad\qquad\qquad\qquad\qquad
        Emilio~Frazzoli
\thanks{The authors are with the Laboratory for Information and Decision Systems, Massachusetts Institute of Technology, Cambridge, MA.}
}

\markboth{}{{Karaman and Frazzoli}: High-speed Flight in an Ergodic Forest}

\maketitle

\begin{abstract}

Inspired by birds flying through cluttered environments such as dense forests, this paper studies the theoretical foundations of a novel motion planning problem: high-speed navigation through a randomly-generated obstacle field when only the statistics of the obstacle generating process are known {\em a priori}. Resembling a planar forest environment, the obstacle generating process is assumed to determine the locations and sizes of disk-shaped obstacles. When this process is ergodic, and under mild technical conditions on the dynamics of the bird, it is shown that the existence of an infinite collision-free trajectory through the forest exhibits a phase transition. On one hand, if the bird flies faster than a certain critical speed, then, with probability one, there is no infinite collision-free trajectory, i.e., the bird will eventually collide with some tree, almost surely, regardless of the planning algorithm governing the bird's motion. On the other hand, if the bird flies slower than this critical speed, then there exists at least one infinite collision-free trajectory, almost surely. Lower and upper bounds on the critical speed are derived for the special case of a homogeneous Poisson forest considering a simple model for the bird's dynamics. For the same case, an equivalent percolation model is provided. Using this model, the phase diagram is approximated in Monte-Carlo simulations.
This paper also establishes novel connections between robot motion planning and statistical physics through ergodic theory and percolation theory, which may be of independent interest.
\end{abstract}


\section{INTRODUCTION}

Flying or land-based, high-performance robots that can quickly navigate through cluttered environments, such as dense forests and urban canyons, have long been an objective of robotics research~\cite{Lentik:2010ef,Selekwa:2008jy,KwangjinYang:2008uu,Shim:2006uq,Langelaan:2005ua,Shim:2005tu,Brock:1999wm}, although very little could have been established in realizing them so far. 
Yet, nature is home to several species of birds that are capable of quickly flying through cluttered environments such as dense forests, swiftly maneuvering among trees as they are encountered~\cite{Flyingwiththefast:L9TlqcNh}. 

Biologist have studied many species of birds for decades~\cite{Warner:1931uh,Brown:1948vv,Brown:1953uu}, establishing a good understanding of their behavior during various activities, such as foraging, migration, and food transport~\cite{Hedenstrom:1995hz}. In particular, the optimal speed at which birds should fly during steady flight, e.g., to minimize the dissipated energy, has been studied extensively~\cite{Hedenstrom:1995hz,Ellington:1991vs,Tobalske:2003kl,Biewener:1998wl,Hedrick:2007ht,Hedrick:2007ex}. 
However, their flight through cluttered environments has received relatively little attention among biologists, although many bird species inhabit dense forests. 
In a recent paper, Hedrick and Biewener argued that the historical focus on steady flight may be due to its theoretical and experimental tractability, rather than its importance~\cite{Hedrick:2007ht,Hedrick:2007ex}. 

Inspired by various applications in robotics and biology, this paper studies the theoretical foundations of a novel motion planning problem: high-speed navigation through a randomly-generated obstacle field with known statistics. Throughout the paper, we motivate this problem by a bird's flight in a densely cluttered planar forest environment. We represent the forest by a suitable spatial stochastic point process, also called the forest-generating process, which determines the locations and the sizes of the trees. We model the bird as a dynamical system described by an ordinary differential equation parametrized by a speed variable. 
Then, we ask the following question: {\em what is the maximum speed at which flight can be maintained indefinitely with probabilistic guarantees, e.g., almost surely?}

In the case in which the forest-generating process is ergodic, the answer turns out to be tied to a novel phase transition result, namely: there exists a critical speed such that, on one hand, for any speed above it there is no infinite collision-free trajectory, with probability one; on the other hand, for any speed below it, there exists at least one 
infinite collision-free trajectory, with probability one. 
In this context, an infinite trajectory is one along which the distance of the bird from its starting point grows unbounded.
In other words, a trajectory is infinite, if it can not be contained in any bounded subset of the plane. Intuitively, on an infinite trajectory the bird keeps exploring new territories in the forest.

Roughly speaking, our assumption of ergodicity implies the absence of long-distance dependencies in the locations and the sizes of the trees. Further implications of the ergodicity assumption are discussed in detail later in the paper.
It is also shown that this assumption is essential for the phase transition results to hold. More precisely, there exists a forest-generating process that is stationary, but not ergodic, such that the probability that there exists an infinite collision-free trajectory through this forest is strictly between zero and one.

For the special case in which the locations of the trees are generated by a homogeneous Poisson process, and their sizes are same, and assuming a simple model for the bird's dynamics, we also derive explicit lower and upper bounds on the critical speed. The proof techniques employed in deriving such bounds can be extended, in principle, to cover a much larger class of forest-generating processes and bird dynamics, although the calculations would be  more complicated.

It is worth noting at this point that spatial point processes are widely used in the context of forestry~\cite{Husch:2003vx}. In fact, Stoyan~\cite{Stoyan:2000vo} noted that no other application area ``uses point process methods so intensively, and stimulated the theory so much, as has forestry." 
In particular, the Poisson process has been used extensively to represent the distribution of trees in a forest stand~\cite{Stoyan:2000vo}. Indeed, Tomppo found that around 30\% of the inventory plots in Finland could be considered as a realization of a spatial Poisson point process~\cite{Tomppo:1986ux}. 
Apart from Poisson processes, forestry literature has also employed vairous cluster processes~\cite{Boyden:2005bla}, Cox processes of several kinds~\cite{moller.ea:scan_jour_stat98}, and Gibbs processes~\cite{Stoyan:1998tp}, most of which are already ergodic or have ergodic variants.

The contributions of this paper can be listed as follows. First, it is shown that, under mild technical assumptions on the dynamics governing the motion of the bird, the existence of infinite collision-free trajectories traversing an ergodic forest exhibits a phase transition.
That is, for such forests there exists a critical speed, $\speed_\mathrm{crit}$, such that the following hold: (i) for any speed $\speed < \speed_\mathrm{crit}$, collision-free flight with speed $\speed$ can be maintained indefinitely, while increasingly getting away from some initial condition, with probability one; (ii)  for any speed $\speed > \speed_\mathrm{crit}$, the bird will eventually collide with a tree, almost surely, despite complete knowledge of all trees and regardless of the algorithm used to plan the bird's motion. 
Second, the special case in which the bird is governed by a simple dynamic model and the forest is generated by a homogeneous spatial Poisson point process with intensity $\rho$, and is such that all trees have a fixed radius $\Radius$ is thoroughly studied. 
Both lower and upper bounds for the critical speed are derived for any pair of $\rho$ and $\Radius$. 
It is shown that, when the bird flies with a speed that is above the upper bound, the probability that the bird can maintain collision-free flight for at least $T$ time units converges to zero exponentially fast with increasing $T$. On the other hand, when the bird flies with a speed that is below the lower bound, then the probability that there exists an infinite collision-free trajectory starting from a location that is within a distance $l$ to the origin converges to one exponentially fast with increasing $l$.
Third, for the same special case, an equivalent novel percolation model is given, and some unique properties of the model are discussed. Fourth, the phase diagram is approximately constructed through Monte-Carlo simulations.

Moreover, the results presented in this paper lay down the theoretical foundations of a novel class of motion planning problems involving high-speed motion in a randomly-generated obstacle field, where the statistics of the obstacle generation process are known, but, for instance, the precise location and the shape of the obstacles are not known {\em a priori}. 
This paper explores the fundamental limits of planning algorithms tailored to solve this problem.
Our analysis implies strong negative results showing that under certain conditions collision-free motion at speeds above a critical speed is impossible to maintain indefinitely, with probability one. 

Finally, this paper establishes novel connections between robot motion planning and statistical physics, in particular ergodic theory~\cite{Walters:2000vc} and percolation theory~\cite{Grimmett:1999ur,Meester:1996ue,Bollobas:2006ur}, which have been used to study many scientific problems in a diverse set of fields~\cite{Sahimi:1994wt} including condensed matter physics~\cite{Halbrutter:1989th}, materials science~\cite{Cahn:1997vf}, social networks~\cite{Castellano:2009ce,Solomon:2000td}, traffic congestion~\cite{Angel:2005vm}, political science~\cite{Galam:2003dv}, ad-hoc communication networks~\cite{Franceschetti:2007tt}, the Internet~\cite{Cohen:2000uw}, economics~\cite{Cowan:1997wz}, and finance~\cite{Stauffer:1999uc}.
Although statistical physics has been used for studying animal behavior before, e.g., L\'evy-type random walks have been the driving force behind optimal foraging theory~\cite{Bartumeus:2009fv}, to the best of our knowledge, our application of ergodic theory and percolation theory in the context of both biology and robotics is novel. Moreover, this application constitutes one of the rare examples where the same theory is used not in the context of science, e.g., for understanding physical, social, economic, or political phenomena, but rather in the context of engineering, e.g., to design high-performance robotic vehicles.

This paper is organized as follows. In Section~\ref{section:bird_forest}, the forest-generating process and the dynamics governing the bird are described, and a formal problem definition is provided. The main result on the phase transitions in ergodic forests is stated and proven in Section~\ref{section:zero_one_laws}. In Sections~\ref{section:subcritical_poisson_regime} and \ref{section:supercritical_poisson_regime}, the Poisson forest model is studied. In particular, lower and upper bounds on the critical speed are derived using discrete and continuum percolation theory. 
In Section~\ref{section:equivalent_model}, an equivalent percolation model for high-speed navigation in a Poisson forest is provided.
In Section~\ref{section:computational}, using this model, the phase diagram is approximately computed experimentally verifying our theoretical bounds on the critical speed. 
Finally, the paper is concluded with remarks in Section~\ref{section:conclusion}.
A preliminary version of this paper is submitted as a conference paper~\cite{Karaman:2011td}.

\section{The Forest, The Bird, and The Problem} \label{section:bird_forest}

In this section, the problem of high-speed motion in a randomly-generated obstacle field is formulated. 
Along the way, fairly general models of the forest-generating process and the bird dynamics are given. An important special case involving a Poisson process generating the locations of the trees and a simple model for the bird's dynamics is introduced.

\subsection{The forest-generating Process}

Throughout the paper, it is assumed that the forest is composed of trees with a circular cross section,  with possibly different random radii. The locations of the trees are assumed to be generated by a spatial point process on the infinite plane, and the radius of each tree is assumed to be determined by a ``mark'' associated with each point of the same process. We postpone a more formal treatment of such point processes until Section~\ref{section:zero_one_laws}. Instead, in this section, we provide a simple definition of forest-generating processes that leads to a problem definition free of sophisticated measure-theoretic concepts. 

Let $\countables{\plane}$ denote the set of all countable subsets of $\plane$. Let $(\Omega, {\cal F}, \PP)$ be a probability space, where $\Omega$ is a sample space, ${\cal F}$ is a sigma-algebra, and $\PP$ is a probability measure. Then, a (spatial) point process is a measurable function $\PointProcess : \Omega \to \countables{\plane}$. Intuitively, a point process assigns a probability measure over the locations of countably many sets of points on the plane. 

Let $\MarkSpace$, also called the {\em mark space}, be a set given along with a suitable topology. A {\em (spatial) marked point process} is a pair $(\PointProcess,\MarkMeasure)$, where $\PointProcess$ is a point process and $\MarkMeasure (y, \cdot)$ is a probability  distribution over the mark space $\MarkSpace$ such that $\MarkMeasure(\cdot, B)$ is a measurable function on $\plane$ for any Borel set $B \subset \MarkSpace$.

A {\em forest-generating process}, usually denoted by $\ForestGeneratingProcess$, is a pair $\ForestGeneratingProcess = (\PointProcess,\RandomTreeRadius)$, where $\RandomTreeRadius = \{\randomTreeRadius{y} : y \in \plane\}$ is a collection of random variables with probability measure $\MarkMeasure(y,\cdot)$ such that $(\PointProcess,\MarkMeasure)$ is a marked point process $(\PointProcess,\MarkMeasure)$ with mark space $\MarkSpace = \reals_{> 0}$. The locations of the trees in the forest are described by $\PointProcess$ whereas the radius of a tree located at $y \in \plane$ is distributed according to the probability measure $\MarkMeasure(y,\cdot)$. 

Notice that this definition of a forest-generating process immediately implies that the radius $r_{y}$ of a tree located $y$ is independent of the locations of all the other trees. 
A more general class of forest-generating processes that take such dependencies into account will be introduced in Section~\ref{section:zero_one_laws} using the notion of counting measures.

\subsection{The Model of the Bird} \label{section:model:bird}

\subsubsection{The dynamics of the bird}
Let $\X \subseteq \reals^n$ and $\U \subset \reals^m$ be measurable sets, where $n , m > 0$ are integers. To model the dynamics governing the motion of the bird, consider the following collection of dynamical systems parametrized by a speed parameter denoted by $\nu \in \reals_{>0}$:
\begin{align}
\begin{array}{c}
\dot{x}(t) = f_{\speedratio} (x(t),u(t)), \quad x(0) = x_0, \\[0.5em]
y(t) = h_\speedratio (x(t)),
\end{array}
\label{eqn:system}
\end{align}
where $f_\speed : \X \times \U \to \reals^n$ and $h_\speed : \X \to \plane$ are Lipschitz continuous in both variables for all $\speed$, and $x(t) \in \X$ and $u(t) \in \U$ for all $t$. A function $y : \reals_{\ge 0} \to \plane$ is said to be a trajectory of the system described by Equation~\eqref{eqn:system} at speed $\speed$ if there exists a measurable function $u : \reals_{\ge 0} \to \U$ such that $y (t)$ and $u(t)$ satisfy Equation~\eqref{eqn:system} for all $t \ge 0$. 

Intuitively, $y(t)$ of is the bird's position in the planar forest at time $t$.
Note, however, that the dynamics of the bird may be fairly complicated involving several state variables. Moreover, the case then the bird is flying in a three-dimensional forest environment is also captured if one assumes that the (cylindrical) trees have no branches.

\subsubsection{The planner governing the bird} 
Recall that $\ForestGeneratingProcess = (\PointProcess,\RandomTreeRadius)$ denotes a forest-generating process.
A realization of the forest-generating process, often denoted by $\ForestGeneratingProcess(\omega) =  (\PointProcess,\RandomTreeRadius)(\omega)$ for a given sample path $\omega \in \Omega$, can be though of as a pair $(\ElementPlane,\Radius) = (\{\ElementPlane_i\}_{i \in \naturals},\{\Radius_i\}_{i \in \naturals})$ of countable sets, where $\ElementPlane_i \in \plane$ and $\Radius_i \in \reals_{>0}$. Let $\ForestRealizations$ denote the set of all such realizations.

A {\em planner} is a pair $\planner = (\plannerControl, \plannerInitState)$, where $\plannerControl : \reals_{\ge 0} \times \ForestRealizations \to U$ and $\plannerInitState : \ForestRealizations \to \plane$.
In this setting, given a realization $(\ElementPlane,\Radius)$ of the forest-generating process, the function $\plannerInitState ((\ElementPlane,\Radius))$ returns an initial state for the bird to start its (constant-speed) flight, and the function $\plannerControl(t; (\ElementPlane,\Radius))$ is the ``motion planner'' that determines the input signal driving the bird through the forest. 

\subsection{Problem Formulation} \label{section:model:problem}

Note that the planner $\planner$ is a random mapping, since it is a function of the realizations of the (random) forest-generating process.\footnote{In general the planner can be randomized, in which case $\plannerControl$ is defined as a probability measure over $\reals_{\ge 0} \times F$ and the sample space is extended suitably.} The {\em input process}, denoted by $\{\Inputvar{\planner}(t) : t \in \reals_{\ge 0}\}$, is a stochastic process, i.e., a collection of random variables indexed by $t$, with the following realizations
$$
\inputvar{\planner}(t, \omega) := \plannerControl (t; \ForestGeneratingProcess(\omega)), \quad \mbox{ for all }\omega \in \Omega. 
$$
Similarly, the {\em state process} and {\em output process}, denoted by $\{\Statevar{\planner}{\speed}(t) : t \ge 0 \}$ and $\{\Outputvar{\planner}{\speed}(t) : t \ge 0\}$, respectively, are the solutions to Equation~\eqref{eqn:system}, when the input $u(t)$ is equal to the input process described above.

Given a realization $(\ElementPlane,\Radius) = (\{\ElementPlane_i\}_{i \in \naturals},\{\Radius_i\}_{i \in \naturals})$ of the forest-generating process, 
define the {\em region occupied by the trees} as 
$$
\TreesGeneral{\ForestGeneratingProcess} := \bigcup_{i \in \naturals} \ball{y_i}{r_i},
$$
where $\ball{y}{r} \subset \plane$ denotes the disk of radius $r$ centered at $y$. Define the {\em free region} by $\FreeGeneral{\ForestGeneratingProcess} := \plane \setminus\TreesGeneral{\ForestGeneratingProcess}$. 

Finally, the following is a formal statement of the problem of high-speed navigation through a randomly-generated obstacle field, specialized to the case of a planar stochastic forest.
\begin{problem} \label{problem:main}
Given the dynamics of the bird described by Equation~\eqref{eqn:system}, a forest-generating process $\ForestGeneratingProcess$, and a speed $\nu$, design a planner $\planner$ such that the bird governed by $\planner$ 
\begin{itemize}
\item[(i)] flies indefinitely towards new regions of the forest with probability one, i.e., 
$$
\PP\Big(\big\{\liminf\nolimits_{t \to \infty} \Vert \Outputvar{\planner}{\speed}(t) \Vert_2 = \infty\big\}\Big) = 1,
$$
\item[(ii)] avoids collision with trees almost surely, i.e.,
$$
\PP\Big(\big\{\Outputvar{\planner}{\speed}(t) \in \FreeGeneral{\ForestGeneratingProcess} \mbox{ for  all } t \in \reals_{\ge 0}\big\}\Big) = 1.
$$
\end{itemize}
\end{problem}

Problem~\ref{problem:main} resembles the classical motion planning problem (see, e.g.,~\cite{LaValle:2006wu}) in the following way.
The first item in the problem definition is a condition that is similar to the requirement of reaching a goal region, which is situated at infinity in this case. The second item, on the other hand, is the same as the usual ``avoid collision with obstacles'' requirement of classical motion planning problems, except with a probabilistic guarantee. 
Thus, Problem~\ref{problem:main} describes the motion planning problem of diverging towards a ``goal region'' situated at infinity while avoiding collision with obstacles almost surely, in a stochastically-generated environment.

In the sequel, a planner is said be {\em reaching}, if it satisfies the first condition in the statement of Problem 1. It is said to {\em almost surely maintain collision-free flight indefinitely at speed $\speed$} if it satisfies the second condition. A planner that satisfies both conditions is said to solve Problem~\ref{problem:main} for speed $\speed$.

In this paper, we report a series of negative results showing the non-existence of planners that solve Problem~\ref{problem:main} for speeds above some critical speed $\speed_\mathrm{crit}$ despite the knowledge of all the trees in the forest, although such planners are guaranteed to exist when the speed is below the very same threshold, whenever the forest-generating process is ergodic (see Section~\ref{section:zero_one_laws}). Moreover, we carry out a thorough analysis of an important special case, in particular by providing upper and lower bounds on the critical speed. This special case is introduced in the next section and analyzed both theoretically in Sections~\ref{section:subcritical_poisson_regime}, \ref{section:supercritical_poisson_regime}, and \ref{section:equivalent_model} and experimentally in Section~\ref{section:computational}.

Throughout the paper, we tacitly assume that the planner has the knowledge of the locations and the sizes of all the trees in the forest once the forest is realized. Since we are mainly concerned with negative results in this paper, this assumption is not limiting, but rather leads to stronger negative results. 

It is also worth noting that, under this assumption, Problem~\ref{problem:main} reduces to the almost-sure existence of infinite collision-free trajectories in the following sense. A trajectory $y : [0,\infty) \to \plane$ is said to be {\em infinite} if $\liminf_{t \to \infty} \Vert y(t) \Vert_2 = \infty$, i.e., the trajectory diverges towards infinity increasingly getting away from any initial condition $y(0)$. The trajectory $y$ is said to be {\em collision-free} for a particular realization $\ForestGeneratingProcess(\omega)$ of the forest-generating process, if $y(t) \in \FreeGeneral{\ForestGeneratingProcess(\omega)}$ for all $t \in \reals_{\ge 0}$. 
Then, clearly, the Ê(non-)existence of planners that solve Problem~\ref{problem:main} for a particular speed is equivalent to the almost-sure \mbox{(non-)}existence of infinite collision-free trajectories at the same speed, when the bird is aware the locations and the sizes of all the trees in the forest once the forest is realized.\footnote{In a more general setting, the planner may be constrained, for instance, with the knowledge of nearby trees only, i.e., the bird may have a limited perception range, which can be formulated by slightly modifying our problem definition to impose a certain measurability condition on the planner $\planner = (\plannerControl, \plannerInitState)$ with respect to a suitable filtration. In that case, the existence of an infinite collision-free trajectory does not necessarily guarantee the existence of a planner that can navigate with a limited perception range.}

From now on, we formulate all our major results in terms of the existence or non-existence of infinite collision-free trajectories and subsequently discuss their implications in the context of Problem~\ref{problem:main}.

\subsection{A Single-integrator Bird Flying in a Poisson Forest} \label{section:model:sing_bird_poisson_forest}

Recall that the dynamics of the bird was described by Equation~\eqref{eqn:system}. Consider the following special case:
\begin{align}
\begin{array}{c}
\dot{x}(t) = f_\nu(x(t), u(t)) =  \left(\begin{array}{c} \nu \\ u(t) \end{array}\right), \\[0.75em]
y(t) =  h_\nu(x(t)) = x(t),
\end{array} \label{eqn:sing_bird}
\end{align}
where $x(t), y(t) \in \plane$ and $u(t) \in [-\wmax, \wmax]$ for all $t \ge 0$. According to this model, the bird is flying at a constant speed, denoted by $\nu$, in the longitudinal direction, while it can maneuver with bounded speed (but unbounded acceleration) in the lateral direction. 
From now on, without loss of any generality, we assume that $\wmax = 1$. 

The equation describing the dynamics prescribed by Equation~\eqref{eqn:sing_bird} is parametrized by $\nu$, which denotes the ``speed'' parameter, although in this case $\speed$ is not precisely the speed of the bird, but rather its speed in the longitudinal direction.

The set of states reachable at or before time $t$ for this system is the cone-shaped region; see Figure~\ref{figure:single_integrator_bird_reachable_set}. 
An important parameter in the analysis of this model turns out to be the aperture of this cone, denoted by $\Maneuverability$, defined as $\Maneuverability := 2 \, \tan^{-1}(1/\speed)$. Roughly speaking, $\Maneuverability$ corresponds to the ``maneuverability'' of the bird; increasing values of $\Maneuverability$ corresponds to larger reachable sets, thus more maneuvering ability for the bird. 

Given a spatial point process $\PointProcess$ and a Borel set $\SubsetPlane \subset \plane$, with a slight abuse of notation, let $\PointProcess(\SubsetPlane)$ denote the number of points that fall into $\SubsetPlane$. Then, $\PointProcess$ is said to be a {\em (homogeneous) Poisson process} with intensity $\lambda$, if (i) for all pairwise disjoint Borel sets $\SubsetPlane_1, \SubsetPlane_2, \dots, \SubsetPlane_n$, the random variables $\PointProcess(\SubsetPlane_1), \PointProcess(\SubsetPlane_2), \dots, \PointProcess(\SubsetPlane_n)$ are mutually independent, and (ii) for every bounded Borel set $\SubsetPlane$, the random variable $\PointProcess(\SubsetPlane)$ is a Poisson random variable with mean $\lambda \, \mu(\SubsetPlane)$, where $\mu(\cdot)$ is the usual Lebesgue measure. A Poisson point process with intensity $\lambda$ is denoted by $\PoissonProcess{\lambda}$.

In Sections~\ref{section:subcritical_poisson_regime}-\ref{section:computational}, we focus on the special case when the bird described by Equation~\eqref{eqn:sing_bird} is flying in a forest generated by the process $(\PoissonProcess{\rho}, \RandomTreeRadius)$, where $\PoissonProcess{\rho}$ is a Poisson point process with intensity $\rho$ and $\RandomTreeRadius$ is equal to some constant, say $r \in \reals_{>0}$.
Throughout the paper, this forest-generating process is called the {\em Poisson forest-generating process} with tree density $\rho$ and the tree radius $r$, and the bird model given by Equation~\eqref{eqn:sing_bird} is called the {\em single-integrator bird} flying at speed $\speed$.

\begin{figure}[b]
\centering
\includegraphics[width = 3cm]{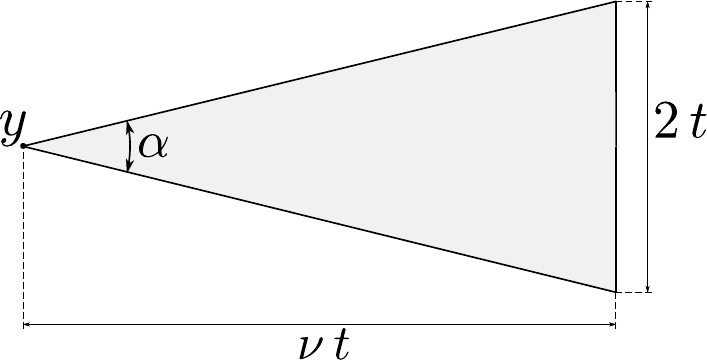}
\caption{The set of states reachable from $y \in \plane$ within time $t$.}
\label{figure:single_integrator_bird_reachable_set}
\end{figure}

\section{Phase Transitions in Ergodic Forests} \label{section:zero_one_laws}

In this section, first the notion of an ergodic forest-generating process is formalized. Subsequently, it is shown that the existence of an infinite collision-free trajectory exhibits a phase transition if the forest-generating process is ergodic.

\subsection{Ergodic forest-generating Processes}

Recall that a forest-generating process was defined as a pair $(\PointProcess, \RandomTreeRadius)$, where $\PointProcess$ is a point process generating the locations of the trees and $\RandomTreeRadius$ represents a family $\{\RandomTreeRadius_z\}_{z \in \plane}$ of random variables, indexed by $z$, that assigns each tree a (random) radius.
This definition imposes the condition that the radius of each tree depends only on the location of that particular tree, but not the locations of any of the others. Such a definition fails to model an important class of realistic forests, e.g., those in which a particular tree may have a larger radius if there are not many other trees around it. 
In this section, the definition of the forest-generating process is generalized to capture such cases using the theory of marked point processes. Along the way, point processes are defined in terms of their counting measures rather than functions from a sample space to a countable subset of the plane, which was adopted in Section~\ref{section:bird_forest} for simplicity. 
Finally, the ergodicity assumption for point processes is also formulated, and its implications are discussed. 

Let $\BorelSets{\reals}$ denote the set of all Borel subsets of $\reals$. The set $\BorelSets{\plane}$ is defined similarly. 
A {\em Borel measure} on $\plane$ is a function, usually denoted by $\BorelMeasure$, that maps $\BorelSets{\plane}$ to $\reals$ such that it is nonnegative, i.e., $\BorelMeasure(A) \ge 0$ for all $A \in \BorelSets{\plane}$, and $\sigma$-additive, i.e., $\BorelMeasure(\cup_{i \in \naturals}A_i) = \sum_{i \in\naturals}\BorelMeasure(A_i)$ for any sequence $\{A_i\}_{i \in \naturals}$ of disjoint Borel sets.
A Borel measure $\BorelMeasure$ on $\plane$ is said to be {\em boundedly finite} if $\BorelMeasure(A) < \infty$ for all bounded $A \in \BorelSets{\plane}$.
A {\em counting measure} $\CountingMeasure$ is a boundedly finite Borel measure that is integer valued, i.e., $\CountingMeasure(A) \in \mathbb{Z}$ for all $A \in \BorelSets{\plane}$. The space of all counting measures on $\plane$ is denoted by $\SpaceCountingMeasure$. The set of all Borel subsets of $\SpaceCountingMeasure$ is denoted by $\BorelSets{\SpaceCountingMeasure}$.

A (spatial) {\em point process} $\PointProcess$ is a measurable mapping from a probability space $(\Omega, {\cal F}, \PP)$ to the space $\SpaceCountingMeasure$ of all counting measures on $\plane$, i.e., for all $\EventCountingMeasure \in \BorelSets{\SpaceCountingMeasure}$, $\PointProcess^{-1}(\EventCountingMeasure) := \{ \omega \in \Omega : \PointProcess(\omega) \in \EventCountingMeasure\} \in {\cal F}$.
Defined in this way, $\PointProcess$ assigns a counting measure $\CountingMeasure = \PointProcess(\omega) \in \SpaceCountingMeasure$ to any sample path $\omega \in \Omega$, such that $\CountingMeasure(A)$ is the number of points that fall into $A \in \BorelSets{\plane}$. 
Thus, $\PointProcess$ randomly generates a counting measure that counts the number of points in any given Borel subset of the plane. 
Given any Borel set $\SubsetPlane \in \BorelSets{\plane}$, the number of points that fall into $A$ in the point process $\PointProcess$, denoted by $\PointProcess(\SubsetPlane)$ with a slight abuse of notation, is an integer-valued random variable.\footnote{More precisely, $(\PointProcess(\cdot))(\SubsetPlane) : \Omega \to \reals$ defines a random variable by mapping the sample space $\Omega$ to the set of real numbers in a measurable way. To indicate this random variable we simply write $\PointProcess(\SubsetPlane)$, while to indicate the counting measure for the sample path $\omega\in \Omega$, we write $\PointProcess(\omega)$, when there is no ambiguity in notation. }

Let $\MarkSpace$ be a mark space endowed with a suitable topology. Let $\SpaceMarkedCountingMeasure$ denote the set of counting measures on the space $\plane \times \MarkSpace$ of point-mark pairs. 
Then, a (spatial) {\em marked point process} $\MarkedPointProcess$ is a measurable mapping from a probability space $(\Omega, {\cal F}, \PP)$ to $\SpaceMarkedCountingMeasure$ such that the {\em ground measure}, defined by 
$$
\GroundPointProcess{\MarkedPointProcess} (\SubsetPlane) = \MarkedPointProcess(\SubsetPlane \times \SubsetMark) \mbox { for all } \SubsetPlane \in \BorelSets{\plane} 
$$
is a boundedly finite counting measure for all $\SubsetMark \in \BorelSets{\MarkSpace}$.
Defined in this way, a marked point process assigns a counting measure $\MarkedCountingMeasure = \MarkedPointProcess (\omega) \in \SpaceMarkedCountingMeasure$ such that $\MarkedCountingMeasure(\SubsetPlaneMark)$ denotes the the number of points $\ElementPlane$ with mark $\ElementMark$ such that $(\ElementPlane,\ElementMark) \in \SubsetPlaneMark$, for all $\SubsetPlaneMark \in \BorelSets{\plane \times \MarkSpace}$.
In particular, $\MarkedCountingMeasure(\SubsetPlane \times \SubsetMark)$ denotes the number of points in $\SubsetPlane \in \BorelSets{\plane}$ with marks in $\SubsetMark \in \BorelSets{\MarkSpace}$. 

Finally, a {\em forest-generating process}, usually denoted by $\ForestGeneratingProcess$, is a marked point process defined on a mark space $\MarkSpace = \reals_{>0}$, where marks denote the radii of the trees.

For all $\TranslateVector, \ElementPlane \in \plane$, all $\ElementMark \in \MarkSpace$, and all $\SubsetPlaneMark \in \BorelSets{\plane \times \MarkSpace}$, define the {\em translation operator} $\Translate{\TranslateVector}{ }$ as $\Translate{\TranslateVector}{(\ElementPlane,\ElementMark)} := (\ElementPlane + \TranslateVector, \ElementMark)$ and $\Translate{\TranslateVector}{\SubsetPlaneMark} := \{(\ElementPlane + \TranslateVector, \ElementMark) : (\ElementPlane,\ElementMark) \in \SubsetPlaneMark\}$. 
Intuitively, given a set $\SubsetPlaneMark$ of marked points in the plane, the translation operator $\Translate{\TranslateVector}{ }$ applied to $\SubsetPlaneMark$ translates the location of all points by a vector $\TranslateVector$ while keeping their marks the same. 
The operator $\Translate{\TranslateVector}{ }$ induces a transformation $\TranslateMeasure{\TranslateVector}{ }$ on $\SpaceMarkedCountingMeasure$, defined as $(\TranslateMeasure{\TranslateVector}{\MarkedCountingMeasure})(\SubsetPlaneMark) = \MarkedCountingMeasure(\Translate{\TranslateVector}{\SubsetPlaneMark})$ for all $\TranslateVector \in \plane$, all $\SubsetPlaneMark \in \BorelSets{\plane \times \MarkSpace}$, and all $\MarkedCountingMeasure \in \SpaceMarkedCountingMeasure$.
Clearly, the space $\SpaceMarkedCountingMeasure$ of all counting measures on $\plane \times \MarkSpace$ is closed under this transformation, i.e., $\TranslateMeasure{\TranslateVector}{\MarkedCountingMeasure} \in \SpaceMarkedCountingMeasure$ for all $\MarkedCountingMeasure \in \SpaceMarkedCountingMeasure$.
Thus, this transformation is naturally generalized to any marked point process by defining $(\TranslateMeasure{\TranslateVector}{\MarkedPointProcess})(\omega) = \TranslateMeasure{\TranslateVector}{(\MarkedPointProcess(\omega))}$ for all $\omega \in \Omega$, which makes $\TranslateMeasure{\TranslateVector}{\MarkedPointProcess}: \Omega \to \SpaceMarkedCountingMeasure$ a measurable mapping for any $\TranslateVector \in \plane$. Hence, the translated marked point process, $\TranslateMeasure{\TranslateVector}{\MarkedPointProcess}$, is a marked point process on its own right.

Every marked point process $\MarkedPointProcess$ induces a probability measure on the measure space $(\SpaceMarkedCountingMeasure, \BorelSets{\SpaceMarkedCountingMeasure})$, defined as $\MarkedPointProcessMeasure{\MarkedPointProcess}(\EventMarkedCountingMeasure) := \PP (\MarkedPointProcess^{-1}(\EventMarkedCountingMeasure))$ for all $\EventMarkedCountingMeasure \in \BorelSets{\SpaceMarkedCountingMeasure}$. The probability measure $\MarkedPointProcessMeasure{\MarkedPointProcess}$ is called the {\em distribution} of $\MarkedPointProcess$. 
A marked point process $\MarkedPointProcess$ is said to be {\em stationary} if its distribution is invariant under the set $\{\TranslateMeasure{\TranslateVector}{ } : \TranslateVector \in \plane \}$ of transformations, i.e., $\MarkedPointProcessMeasure{\MarkedPointProcess} (\EventMarkedCountingMeasure) = \MarkedPointProcessMeasure{\TranslateMeasure{\TranslateVector}{\MarkedPointProcess}} (\EventMarkedCountingMeasure)$ for all $\EventMarkedCountingMeasure \in \BorelSets{\SpaceMarkedCountingMeasure}$. In other words, $\MarkedPointProcess$ is said to be stationary if the family $\{ \TranslateMeasure{\TranslateVector}{\MarkedPointProcess} : \TranslateVector \in \plane \}$ of marked point processes have the same distribution. 
Finally, a stationary marked point process $\MarkedPointProcess$ is said to be {\em ergodic} if 
$$
\lim_{a \to \infty} \frac{1}{a^{2}} \int_{[0,a]^{2}} \MarkedPointProcessMeasure{\MarkedPointProcess} ((\TranslateMeasure{\TranslateVector}{\EventMarkedCountingMeasure_1}) \cap \EventMarkedCountingMeasure_2) d\TranslateVector = \MarkedPointProcessMeasure{\MarkedPointProcess} ( \EventMarkedCountingMeasure_1) \, \MarkedPointProcessMeasure{\MarkedPointProcess} ( \EventMarkedCountingMeasure_2),
$$
for all $\EventMarkedCountingMeasure_1, \EventMarkedCountingMeasure_2 \in \BorelSets{\SpaceMarkedCountingMeasure}$.
A forest-generating process is {\em ergodic}, if its corresponding marked point process is ergodic.

Roughly speaking, an ergodic marked point process is one for which any two events are ``almost independent'' whenever they can be described by locations and marks of two sets of points such that the points in different sets are far away from each other. 
To provide some examples, first note that the definitions of the translation operator, distribution, stationarity, and ergodicity naturally specialize to point processes (that are not marked). The interested reader is referred to~\cite{Daley:2007wu} for the details.
Then, for example, a (homogeneous) Poisson process is an ergodic point process trivially satisfying the equation above. Moreover, several types of Cox processes and a variety of cluster processes are known to be ergodic (see, e.g.,~\cite{Daley:2007wu}).

There are many examples of ergodic marked point processes also. 
Firstly, if the all the trees have the same radius, then clearly a forest-generating process is ergodic if and only if the locations of the trees are generated by an ergodic point process.
A more general class of ergodic forest-generating processes consists of those in which the locations of the trees are generated by an ergodic point process and their radii are independent and identically distributed random variables~\cite{Meester:1996ue}, a model that is widely used in the context of forestry~\cite{Husch:2003vx}.
However, the class of all ergodic forest-generating processes is much larger, in particular allowing ``short-range dependencies'' between the radii of the trees (see, e.g.,~\cite{Smythe:2005uf}).

Yet, not every stationary forest-generating process is ergodic. A canonical example is the one inÊ which all trees have the same radius and their locations are generated by a mixed Poisson processes that has intensity $\rho_{1}$ with probability $p$ and $\rho_{2}$ with probability $1-p$. This forest-generating process fails to be ergodic whenever $p \in (0,1)$ (see, e.g.,~\cite{Meester:1996ue} for a proof).

Ergodic point processes, marked or unmarked, admit an important characterization through their invariant sigma-algebra being trivial. In this context, an event $\EventMarkedCountingMeasure \in \BorelSets{\SpaceMarkedCountingMeasure}$ is said to be {\em invariant} under the transformation $\TranslateMeasure{u}{ }$, if $\InverseTranslateMeasure{u}{\EventMarkedCountingMeasure} = \EventMarkedCountingMeasure$. It can be shown that the set $\MarkedInvariantSigmaAlgebra$ of all events in $\BorelSets{\SpaceMarkedCountingMeasure}$ that are invariant under the family $\{\TranslateMeasure{u}{ } : u \in \plane\}$ of transformations is indeed a $\sigma$-algebra~\cite{Daley:2007wu}. The set $\MarkedInvariantSigmaAlgebra$ is said to be {\em trivial} if $\MarkedPointProcessMeasure{\MarkedPointProcess} (\EventMarkedCountingMeasure) \in  \{0,1\}$  for all $\EventMarkedCountingMeasure \in \MarkedInvariantSigmaAlgebra$.
The following theorem is a central result in the theory of ergodic point processes.
\begin{theorem}[\cite{Daley:2007wu}] \label{theorem:invariant-sigma-algebra}
A marked point process $\MarkedPointProcess$ is ergodic if and only if $\MarkedInvariantSigmaAlgebra$
, the set of events invariant under transformations $\{\TranslateMeasure{\TranslateVector}{ } : \TranslateVector \in \plane\}$, 
is trivial under the distribution of $\MarkedPointProcess$.
\end{theorem}

\subsection{The Monotonic Zero-One Law of the Ergodic Forest} \label{section:phase_transition:zero_one_law}

In this section, the first main result of this paper is stated and proven.
Before stating the  result, a set of useful definitions are noted, and two intermediate results are established.

Recall that Equation~\eqref{eqn:system} describes the dynamics of the bird. 
A trajectory $y : [0,\Time] \to \plane$, where $\Time > 0$, is said to be {\em dynamically feasible} at speed $\speed$, if there exists $x: [0,\Time] \to \reals^n$ and $u: [0,\Time]\to \reals^m$ such that $u,x,y$ satisfy Equation~\eqref{eqn:system} for all $t \in [0,\Time]$ and some $x_0 \in \reals^n$.
Let $\SetTrajectories{\speed}$ denote the family of all dynamically feasible trajectories at speed $\nu$.

The dynamics of the is said to be {\em translation invariant}, if any dynamically feasible trajectory is still dynamically feasible when translated as a whole. More precisely, the the bird has translation-invariant dynamics, if for all $\Trajectory \in \SetTrajectories{\speed}$ and all $\ElementPlane_0 \in \plane$, the translated trajectory $\Trajectory'$, defined by $\Trajectory'(t) := \Trajectory(t) + \ElementPlane_0$ for all $t \in [0,\Time]$, satisfies $\Trajectory' \in \SetTrajectories{\speed}$.
Roughly speaking, this assumption implies that the dynamics of the bird does not depend on a particular location in the forest. 

This section is concerned with the existence of an infinite collision-free trajectory for the bird described by the event
\begin{align*}
\Event_\ForestGeneratingProcess(\speed) &  := \big\{ \mbox{there exists a trajectory } y \in \SetTrajectories{\speed} \mbox{ such that } \\ & \lim_{t \to \infty} \Vert y(t) \Vert_2 = \infty\mbox{ and }   y(t) \in \FreeGeneral{\ForestGeneratingProcess} \mbox{ for all }t \in \reals_{>0} \big\}.
\end{align*}

The following intermediate result states that a certain type of zero-one law holds if the forest-generating process is ergodic.

\begin{lemma} \label{theorem:ergodic_forest:zero_one_law}
Let $\ForestGeneratingProcess$ be a forest-generating process and suppose that the dynamics governing the bird is translation invariant.
Then, $\PP(\Event_\ForestGeneratingProcess(\speed)) \in \{0,1\}$ for all $\speed$, whenever the forest-generating process is ergodic.
\end{lemma}
\begin{proof}
By Theorem~\ref{theorem:invariant-sigma-algebra}, it is enough to show that the event $\Event_{\ForestGeneratingProcess}(\speed)$ is invariant under the family $\{\TranslateMeasure{\TranslateVector}{ }: \TranslateVector \in \plane\}$ of transformations, i.e., $\TranslateMeasure{\TranslateVector}{ }^{-1} \Event_{\ForestGeneratingProcess}(\speed) = \Event_{\ForestGeneratingProcess}(\speed)$ for all $\TranslateVector \in \plane$.

Fix some $\TranslateVector \in \plane$. For any sample path $\omega \in \Event_{\ForestGeneratingProcess}(\speed)$, by the definition of $\Event_{\ForestGeneratingProcess}(\speed)$, there exists an infinite trajectory $y \in \SetTrajectories{\speed}$ that satisfies $y(t) \in \FreeGeneral{\ForestGeneratingProcess(\omega)}$ for all $t \in \reals_{\ge 0}$. Whenever the locations of all the trees are shifted by a vector $\TranslateVector$ without changing their radii, i.e., the transformation $\TranslateMeasure{\TranslateVector}{ }$ is applied to $\ForestGeneratingProcess$, using the translation invariance of the dynamics, one can construct a new infinite trajectory $y' \in \SetTrajectories{\speed}$, defined as $y'(t) := y(t) + \TranslateVector$ for all $t \in \reals_{\ge 0}$. Clearly, $y'$ satisfies $y'(t) \in \FreeGeneral{\ForestGeneratingProcess(\TranslateMeasure{\TranslateVector}{\omega})}$ for all $t \ge 0$. 
Thus, $\TranslateMeasure{\TranslateVector}{ }\omega \in \Event_{\ForestGeneratingProcess}(\speed)$, or alternatively $\omega \in \TranslateMeasure{\TranslateVector}{ }^{-1} \Event_{\ForestGeneratingProcess}(\speed)$. Hence, $\Event_{\ForestGeneratingProcess}(\speed) \subseteq \TranslateMeasure{\TranslateVector}{ }^{-1} \Event_{\ForestGeneratingProcess}(\speed)$.
A similar argument shows $\TranslateMeasure{\TranslateVector}{ }^{-1}\Event_{\ForestGeneratingProcess} (\speed) \subseteq \Event_{\ForestGeneratingProcess}(\speed)$ also.
\end{proof}

The bird is said to have {\em non-decreasing path sets with decreasing speed}, if any dynamically feasible trajectory can be retraced, albeit with a reparametrization of time, at a lower speed. That is, for any $\speed, \speed' \in \reals_{>0}$ with $\speed' < \speed$, and any $\Trajectory \in \SetTrajectories{\speed}$, thee exists a continuous function $\Reparametrization : \reals_{\ge 0} \to \reals_{\ge 0}$ such that the time-reparametrized trajectory $\Trajectory'$, defined by $\Trajectory'(t) := \Trajectory(\Reparametrization(t))$, satisfies $\Trajectory' \in \SetTrajectories{\speed'}$.
Roughly speaking, this assumption implies that the system becomes only more ``maneuverable" with decreasing speed, which is arguably the case for most agile dynamical systems.
A direct consequence of this assumption is stated in the following lemma.
\begin{lemma} \label{theorem:ergodic_forest:monotonicity}
Suppose that the dynamics governing the bird has non-decreasing path sets with decreasing speed. Then, 
$
\PP\big( \Event_\ForestGeneratingProcess(\speed) \big)
$
is a non-increasing function of $\nu$.
\end{lemma}
\begin{proof}
Let $\speed, \speed' \in \reals_{>0}$ be two speeds such that $\speed<\speed'$.
It is enough to show that $\Event_{\ForestGeneratingProcess}(\nu') \subseteq \Event_{\ForestGeneratingProcess}(\nu)$, which implies $\PP(\Event_{\ForestGeneratingProcess}(\nu')) \le \PP(\Event_{\ForestGeneratingProcess}(\nu))$ by the monotonicity of probability measures.
Consider any sample path $\omega \in \Event_{\ForestGeneratingProcess}(\speed')$. Then, there exists an infinite trajectory $y' \in \SetTrajectories{\speed'}$ such that $y'(t) \in \FreeGeneral{\ForestGeneratingProcess(\omega)}$ for all $t \ge 0$. Since the dynamics of the bird has non-decreasing path sets with decreasing speed, the same trajectory, up to a suitable reparametrization of time, can be generated when the bird is flying slower. That is, there exists an infinite trajectory $y \in \SetTrajectories{\speed}$ such that $y(t) \in \FreeGeneral{\ForestGeneratingProcess(\omega)}$ for all $t \ge 0$. Thus, $\omega \in \Event_{\ForestGeneratingProcess}(\speed)$. 
Hence, $\Event_{\ForestGeneratingProcess}(\nu') \subseteq \Event_{\ForestGeneratingProcess}(\nu)$.
\end{proof}

Finally, the main result of this section%
, which follows directly from Lemmas~\ref{theorem:ergodic_forest:zero_one_law} and \ref{theorem:ergodic_forest:monotonicity}, 
is stated below.

\begin{theorem}\label{corollary:ergodic_forest}
Suppose that the dynamics governing the bird, described by Equation~\eqref{eqn:system}, is translation invariant and it has decreasing reachable sets with increasing speed. 
Then, there exists a critical speed $\speed_\mathrm{crit}$, possibly zero or infinity, such that 
\begin{itemize}
\item for any speed $\speed > \speed_\mathrm{crit}$, there exists no infinite trajectory for the bird that is collision free, with probability one, 
\item for any speed $\speed < \speed_\mathrm{crit}$, there exists at least one infinite collision-free trajectory for the bird, with probability one.
\end{itemize}
\end{theorem}

Theorem~\ref{corollary:ergodic_forest} implies the following regarding the existence of reaching planners that can almost surely maintain high-speed flight indefinitely. On one hand, when the bird flying with speed that is above $\speed_\mathrm{crit}$, it is doomed to crash into a tree eventually, with probability one, under the guidance of any reaching planner despite the knowledge of all the trees in the forest. On the other hand, there exists a planner that can indefinitely maintain collision-free flight with any speed below $\speed_\mathrm{crit}$, almost surely. 
Note that the proof is not constructive in the sense it does not explicitly provide such a planner, but it only shows its existence.
We construct an explicit planner that can navigate a single-integrator bird through a sparse-enough Poisson forest in the next section.

Let us note that assuming stationarity alone does not necessarily lead to the phase transition stated in Theorem~\ref{corollary:ergodic_forest}; the ergodicity assumption is essential. In other words, there exists a forest-generating process that is stationary, but not ergodic, such that the conditions of Theorem~\ref{corollary:ergodic_forest} do not hold. A slightly more general statement is given in the theorem below, the proof of which is delayed until Section~\ref{section:equivalent_model}.
Recall that the dynamics of the single-integrator bird is given by Equation~\eqref{eqn:sing_bird}.

\begin{theorem} \label{theorem:non_ergodic_counter_example}
There exists a family $\{\ForestGeneratingProcess_q : q \in (0,1)\}$ of stationary, but not ergodic, forest-generating processes parametrized by $q$, and a speed $\speed$ such that the probability that there exists an infinite collision-free trajectory for the single integrator bird flying with speed $\speed$ is equal to $q$, i.e., $\PP(\Event_{\ForestGeneratingProcess_q}(\speed)) = q$.
\end{theorem}

\section{Sub-critical Flight in a Poisson Forest} \label{section:subcritical_poisson_regime}

In the previous section, under mild technical assumptions on the dynamics of the bird, it was shown that existence of infinite collision-free trajectories in a randomly-generated forest environment exhibits a phase transition whenever the forest-generating process is ergodic. However, the proof is not constructive in the sense that the threshold speed, denoted by $\speed_\mathrm{crit}$, is neither explicitly computed nor bounded by any means. In this section, we focus on the special case of a single-integrator bird flying in a Poisson forest and compute a lower bound on $\speed_\mathrm{crit}$ using discrete percolation theory. Moreover, we analyze the effect of the starting point on the possibility of high-speed flight for the same case. That is, we try to answer the following question: {\em Does there exist a nearby starting point that leads to an infinite collision-free trajectory?}
We show that the probability that there exists such a starting point within a distance of $l$ to any given point on the plane 
converges to one exponentially fast with increasing $l$.

This section is organized as follows. First, we interrupt the flow of the paper and introduce the theory of discrete percolation, which is used heavily throughout this section. Second, we show that infinite collision-free trajectories exist for the single integrator bird flying in a Poisson forest, almost surely, whenever a certain type of percolation occurs in a novel percolation model that we describe subsequently. Finally, by carefully analyzing this percolation model, we derive a non-trivial lower bound on the critical speed.

\subsection{The Theory of Discrete Percolation}

A canonical model in percolation theory, called the {\em bond percolation on the square lattice}, is described as follows.
Let $\integers$ denote the set of all integers and $\LOne{\cdot}$ denote the usual $L_1$ norm that gives rise to the Manhattan distance.
Consider the graph that has the vertex set $\lattice$ and has an edge between any two vertices $\ElementLattice_1, \ElementLattice_2 \in \lattice$ whenever $\LOne{\ElementLattice_1 - \ElementLattice _2} = 1$. This graph, called the {\em square lattice}, is illustrated in Figure~\ref{figure:lattice}.(a).

Suppose that each edge in the square lattice is declared {\em open} with probability $\Probability$ and {\em closed} otherwise, independently from every other edge. 
A sequence of vertices $(\ElementLattice_0, \ElementLattice_1, \dots, \ElementLattice_k)$ is called an {\em open path}, if $\ElementLattice_i$ and $\ElementLattice_{i+1}$ share an open edge for all $i \in \{1,2,\dots,k\}$. 
An {\em open cluster} is a maximal set of vertices connected by open paths. 
An {\em infinite open cluster} is an open cluster with infinitely many vertices. 
Finally, {\em bond percolation} is said to occur, if the lattice contains an infinite open cluster.

Let $\OpenCluster{p}$ denote the open cluster that contains the origin. See Figure~\ref{figure:lattice}.(b) for an illustration. Note that the choice of origin is irrelevant, since any vertex is statistically indistinguishable from any other vertex. 
Clearly, $\OpenCluster{p}$ is a non-empty set, since it definitely contains the origin itself.
The {\em critical probability for bond percolation} (on the lattice) is defined as the smallest $p$ such that the open cluster that contains the origin has infinitely many vertices with non-zero probability, i.e.,  
$$
\CriticalProbability := \inf\left\{ p \, : \, \PP (\{\Card{\OpenCluster{p}} = \infty\}) > 0 \right\}.
$$

\begin{figure}[b]
\centering
\subfigure[]{\includegraphics[width=0.41\linewidth]{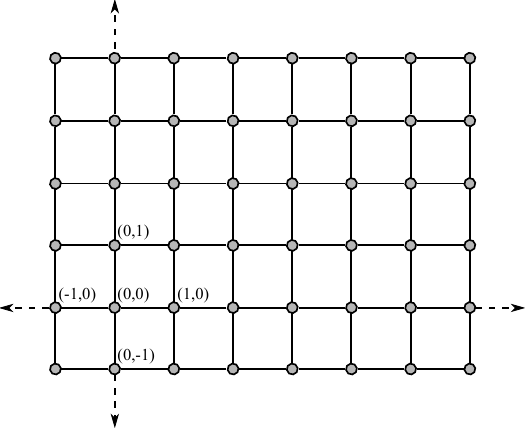}} \qquad
\subfigure[]{\includegraphics[width=0.41\linewidth]{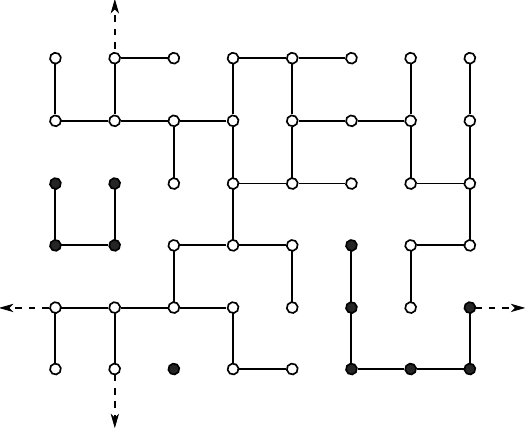}}
\caption{A part of the square lattice is shown Figure (a). Coordinates of the origin and its neighbors are also shown. A particular realization of the bond percolation model is illustrated in Figure (b), where only open edges are drawn. The open cluster containing the origin is shown with white vertices. All other vertices are shown in black.}
\label{figure:lattice}
\end{figure}

Clearly, $\PP (\{\Card{\OpenCluster{p}} = \infty\})$ is equal to zero for $p = 0$ and one for $p =1$. 
It is rather easy to show that $\PP (\{\Card{\OpenCluster{p}} = \infty\})$ is a non-decreasing function of $p$. Moreover, it can be shown that $\CriticalProbability$ takes a non-trivial value, i.e., $\CriticalProbability \in (0,1)$. 
As intuitive as it may seem, it has taken mathematicians twenty years to produce a formal proof that $\CriticalProbability = 1 / 2$ (see, e.g.,~\cite{Bollobas:2006ur}). What is more striking, however, is stated in the following theorem.
\begin{theorem}[\cite{Bollobas:2006ur}] \label{theorem:discrete_percolation}
Consider the bond percolation model on the square lattice where each bond is open with probability $\Probability$ independently of all other bonds. Then, there is no infinite open cluster almost surely when $\Probability < \CriticalProbability$, while a unique infinite open cluster exists with probability one when $\Probability > \CriticalProbability$. 
\end{theorem}

In other words, on one hand, bond percolation fails to occur, almost surely, when $\Probability < \CriticalProbability$; on the other hand, when $\Probability > \CriticalProbability$, not only bond percolation occurs with probability one, but also the infinite open cluster is {\em unique}.

A bond percolation model with $\Probability < \CriticalProbability$ is said to be in the {\em sub-critical regime} whereas one with $\Probability > \CriticalProbability$ is said to be in the {\em super-critical regime}. It is well known that in the sub-critical regime, the probability that the size of a particular open cluster is larger than $l$ is exponentially small in $l$~\cite{Bollobas:2006ur}. A similar statement for the super-critical regime is given below.

Let $\ball{\ElementLattice}{l}$ denote the ball of radius $l$ centered at $\ElementLattice$ defined as $\ball{\ElementLattice}{l} := \{\ElementLattice' \in \lattice : \LOne{\ElementLattice' - \ElementLattice} \le l\}$. Let $\InfiniteOpenCluster{\Probability}$ denote the unique open cluster in the super-critical regime when $\Probability > \CriticalProbability$.
\begin{theorem}[\cite{Bollobas:2006ur}] \label{theorem:sub_exponential_closed_cluster}
Consider the bond percolation model on the lattice with $\Probability > \CriticalProbability$. Then, there exists constants $c_1, c_2 \in \reals_{>0}$ such that for any $\ElementLattice \in \lattice$ and any $l \in \reals_{>0}$, 
$$
\PP(\{\ball{\ElementLattice}{l} \cap \InfiniteOpenCluster{p} = \emptyset\}) \le c_1 \, e^{-c_2 \, l},
$$
\end{theorem}

Another widely-studied discrete percolation model, called the {\em site percolation on the square lattice}, is defined as follows. Suppose each vertex, rather than each edge, of the square lattice is declared open with probability $p$ and closed otherwise. An open path, in this case, is a sequence of open vertices connected by edges. An open cluster and an infinite open cluster are defined similarly. 
Let $\OpenClusterSite{p}$ denote the set of all vertices that are connected to the origin by an open path. This set is also called the open cluster containing the origin. Note that $\OpenClusterSite{p}$ is an empty set if the origin itself is closed.
The {\em critical probability for site percolation} is defined as the smallest $p$ such that the open cluster containing the origin is infinite with non-zero probability, i.e., 
$
\CriticalProbabilitySite  := \inf \{ p\; : \; \PP(\{\Card{\OpenClusterSite{p}} = \infty\}) > 0 \}.
$

Similar to bond percolation, it is rather easy to show that $\CriticalProbabilitySite \in (0,1)$. However, the precise value of $\CriticalProbabilitySite$ is not known, although rigorous bounds are available~\cite{Bollobas:2006ur}. In particular, it has been shown that $\CriticalProbabilitySite > 1/2$. Simulation studies suggest that $\CriticalProbabilitySite$ is around 0.59274598~\cite{Lee:2008kf}.
Again, for all $\Probability < \CriticalProbabilitySite$, percolation fails to occur almost surely, while for all $\Probability > \CriticalProbabilitySite$, a unique infinite open cluster exists with probability one. Moreover, Theorem~\ref{theorem:sub_exponential_closed_cluster} holds also for the site percolation case.

\begin{figure}[b]
\centering
\begin{minipage}[h]{0.57\linewidth}
\includegraphics[width=\linewidth]{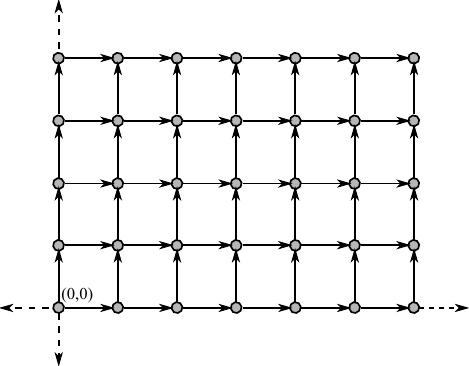}\\
\centerline{\small (a)}
\end{minipage}
\begin{minipage}[h]{0.38\linewidth}
\includegraphics[width=\linewidth]{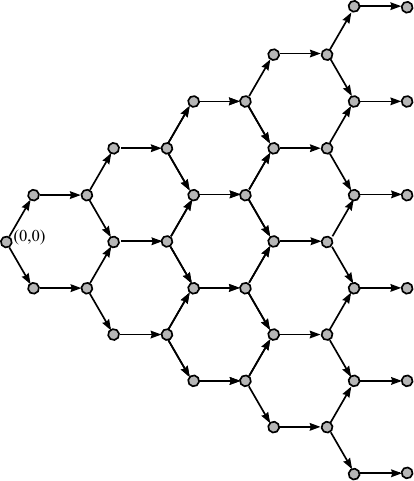}\\
\centerline{\small (b)}
\end{minipage}
\caption{The directed square lattice is shown in Figure (a). The directed hexagonal lattice is shown in Figure (b).}
\label{figure:lattice_directed}
\end{figure}

Percolation models on directed graphs are of particular interest for the purposes of this paper. Consider the directed square lattice with vertices only in the positive orthant and the edges are directed in the direction of the positive axes (see Figure~\ref{figure:lattice_directed}.(a)). Both site percolation and bond percolation models on the directed square lattice are defined similar to their undirected components. In this case, an {\em open path} is a sequence of open vertices connected with open directed edges. The critical probabilities are defined similarly. It was shown that also in the directed case the critical probabilities take non-trivial values, although their precise values are not known~\cite{Bollobas:2006ur}. Simulations suggest that the critical probability for bond percolation is around 0.644700185 whereas that for site percolation is around 0.70548522~\cite{Jensen:1999vu}.
Finally, note that Theorems~\ref{theorem:discrete_percolation} are \ref{theorem:sub_exponential_closed_cluster} are both valid also for both bond percolation and site percolation on the directed square lattice~\cite{Bollobas:2006ur}.

Square lattice is not the only graph on which the percolation phenomenon occurs. In fact, there are a variety graphs, each of which lead to different percolation thresholds. The interested reader is referred to~\cite{Bollobas:2006ur} for an extensive survey. A graph that is important for the purposes of this paper is the directed hexagonal lattice shown in Figure~\ref{figure:lattice_directed}.(b).
The critical probability is known to be non-trivial both for the bond percolation and the site percolation models on the hexagonal lattice~\cite{Bollobas:2006ur}. In particular, the critical probability for bond and site percolation on the hexagonal lattice is known to be at most $(1+\sqrt{33})/8$ and $\sqrt{3}/2$, respectively (see~\cite{Bollobas:2006ur} and the references therein).
Finally, Theorems~\ref{theorem:discrete_percolation} and \ref{theorem:sub_exponential_closed_cluster} are both valid for both the directed site percolation and the directed bond percolation models on the hexagonal lattice (see, e.g.,~\cite{Durrettt:1984tma}).

\subsection{The Existence of Infinite Collision-free Trajectories}

In this section, it is shown that the existence of infinite collision-free trajectories for the single-integrator bird in a Poisson forest is implied by site percolation on a directed hexagonal lattice, which is then used to derive a lower bound for the critical speed under certain conditions.

\begin{figure}[b]
\centering
\begin{minipage}[h]{0.65\linewidth}
\includegraphics[width=\linewidth]{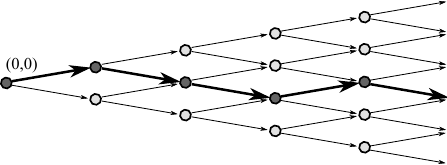}
\end{minipage}\hfill
\begin{minipage}[h]{0.3\linewidth}
\includegraphics[width=\linewidth]{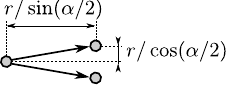}
\end{minipage}
\caption{The set of all output trajectories of the single integrator starting from the origin under some input signal in the aforementioned input set is shown on the left. An example trajectory from this set is shown in bold. The circles are the points in $\OutputSwitchPoints{\speed}$. On a particular edge the input signal, equal to either $-\wmax$ or $\wmax$, is kept constant. On the right, the length traveled in longitudinal and lateral directions is shown when $\Time = \Radius/(\speed \, \sin (\alpha/2))$.}
\label{figure:output_trajectories}
\end{figure}

Recall from Section~\ref{section:model:sing_bird_poisson_forest} that $\speed$ denotes the longitudinal speed of the single-integrator bird and $\Radius$ denotes the radius of each tree in the Poisson forest. Recall also that $\anglespeed := 2\, \tan^{-1}(1/\speed)$. 
Let $\Time = \Radius/(\speed \, \sin (\anglespeed/2))$. 
Define the infinite sequence $(\TimeInterval_0, \TimeInterval_1, \dots)$ of time intervals as $\TimeInterval_k = [k \, T, (k+1)\, T)$ for all $k \in \naturals$. 
Let $\InputSequence = (\w_0, \w_1,\dots)$ be a sequence of lateral speeds for the single integrator bird such that $\w_k \in \{-\wmax, \wmax\}$ for all $k \in \naturals$. Define the input signal $\inputvar{\sigma}$ as $\inputvar{\sigma}(t) = \w_k$ for all $t \in \TimeInterval_k$. Consider  the set of input signals corresponding to the set of all such $\sigma$, i.e., 
$$
\big\{ u_{(\w_0, \w_1, \dots)} \, : \, \w_k \in \{-\wmax, \wmax\} \mbox { for all } k \in \naturals \, \big\}.
$$
Let $\OutputTrajectories{\speed}$ denote the set of all output trajectories of the single-integrator bird that correspond to input signals from the set given above and that start from the origin. Let $\OutputSwitchPoints{\speed}$ denote the set of points that the value of the input signal may switch, i.e., $\OutputSwitchPoints{\speed} := \{y(k \Time) : y \in \OutputTrajectories{\speed}, k \in \naturals\}$. The sets $\OutputTrajectories{\Time}$ and $\OutputSwitchPoints{\speed}$ are depicted in Figure~\ref{figure:output_trajectories}.
Note that the longitudinal and lateral distances traveled by the bird during a time interval of length $\Time$ are $\Radius/\sin(\anglespeed/2)$ and $\Radius/\cos(\anglespeed/2)$, respectively.

The following theorem states the main result of this section. Its proof establishes a key connection between the existence of infinite collision-free trajectories in $\OutputTrajectories{\speed}$ and site percolation on the directed hexagonal lattice.

\begin{theorem} \label{theorem:subcritical_poisson_flight:threshold_bound}
There exists an infinite collision-free trajectory of the single integrator bird, almost surely, whenever the tree density $\rho$, tree radius $\Radius$, and the speed $\nu$ satisfy the following:
$$
\frac{\rho \, \Radius}{\sin (\alpha)} < \log \left( 1/ \sqrt{\CriticalProbabilitySite}\right),
$$
where $\alpha = 2 \, \tan^{-1} (1/\speed)$ and $\CriticalProbabilitySite$ is the critical probability for site percolation on the directed hexagonal lattice. 
\end{theorem}
\begin{proof}
Let $\Trajectory \in \SetTrajectories{\speed}$ be a trajectory of the single-integrator bird.
Since all trees of the Poisson forest-generating process have the same radius, namely $\Radius$, a necessary and sufficient condition for $\Trajectory$ to be collision free is that no points of the Poisson point process falls on the tunnel of width $\Radius$ around $\Trajectory$, defined as $\bigcup_{t \in [0,\infty)} \ball{\Trajectory(t)}{\Radius}$, where $\ball{\Trajectory(t)}{\Radius}$ denotes the Euclidean ball of radius $\Radius$ centered at $\Trajectory(t)$. This region is shown in Figure~\ref{figure:output_trajectory_safe_region} for the particular example trajectory that was depicted in Figure~\ref{figure:output_trajectories}. 

\begin{figure}[b]
\centering
\includegraphics[width=0.8\linewidth]{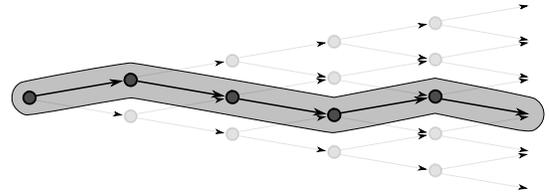}
\caption{The shaded region is the set of all points that should not include a point of the Poisson point process in order for the shown trajectory to be collision free.}
\label{figure:output_trajectory_safe_region}
\end{figure}

\begin{figure}[b]
\centering
\includegraphics[width=0.9\linewidth]{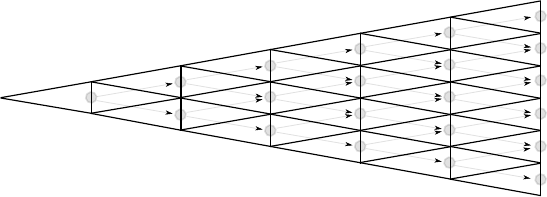}
\caption{Triangular regions associated with the switching points in $\OutputSwitchPoints{\speed}$.}
\label{figure:output_trajectory_model}
\end{figure}

\begin{figure}[t]
\centering
\includegraphics[width=0.9\linewidth]{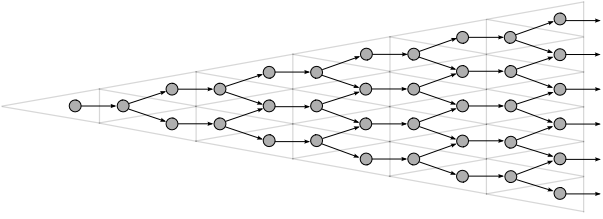}
\caption{The site percolation model on the hexagonal lattice implied by the percolation on the triangular regions associated with the switching points.}
\label{figure:output_trajectory_model_graph}
\end{figure}

This motivates the following percolation model. 
Associate two triangular regions with each point in $\OutputSwitchPoints{\speed}$ as illustrated in Figure~\ref{figure:output_trajectory_model}.
Consider the graph formed by associating a vertex with each triangular region and connecting two vertices with a directed edge if their corresponding triangular regions share an edge. The edge is always directed towards the positive longitudinal axis. See Figure~\ref{figure:output_trajectory_model_graph} for an illustration. 

Declare a vertex {\em open} if its corresponding triangular region contains no point of the Poisson process, and {\em closed} otherwise. The probability of the event that a particular vertex is open is 
$$
\Probability = e^{- \rho \, \Radius^2/(\sin (\alpha/2) \, \cos (\alpha/2))}.
$$

Since the point process is Poisson and the triangular regions are openly disjoint, every vertex is open with probability $\Probability$ independently from every other vertex. Hence, this model is equivalent to directed site percolation on the hexagonal lattice. 

Finally, notice that there exists an infinite collision-free trajectory in $\OutputTrajectories{\speed}$, thus in $\SetTrajectories{\speed}$, whenever there exists an infinite open cluster in this model. 
The latter occurs when 
$$
 e^{- \rho \, \Radius^2/(\sin (\alpha/2) \, \cos (\alpha/2))} > \CriticalProbabilitySite,
$$
where $\CriticalProbabilitySite$ is the critical probability for site percolation on the hexagonal lattice. 
Then, the result follows by taking the natural logarithm of both sides and substituting $\sin (\alpha) / 2$ for $\sin(\alpha/2) \cos(\alpha/2)$.
\end{proof}

In the following, the existence of infinite collision-free trajectories is further investigated. 
Roughly speaking, the following theorem states that, under the conditions of Theorem~\ref{theorem:subcritical_poisson_flight:threshold_bound}, there exists a nearby starting point that leads to an infinite collision-free trajectory with very high probability. 
\begin{theorem} \label{theorem:subcritical_poisson_regime:exponential_decay}
Suppose that the conditions of Theorem~\ref{theorem:subcritical_poisson_flight:threshold_bound} are satisfied and that $\rho \, \Radius^2 / \sin(\Maneuverability) < \log (1/\sqrt{\CriticalProbabilitySite})$. Given $l \in \reals_{>0}$ and $\ElementPlane \in \plane$, let $\Event_\ElementPlane(l)$ denote the event that there exists an infinite collision-free trajectory with speed $\speed$ that starts from some $\ElementPlane'$ with $\Vert \ElementPlane' - \ElementPlane \Vert \le l$. Then, there exists constants $c_1, c_2 \in \reals_{>0}$ such that for any $\ElementPlane \in \plane$ and for all $l \in \reals_{>0}$,
$$
\PP\big(\,\Event_\ElementPlane (l)\,\big) \; \ge \; 1 - c_1 \, e^{-c_2 \, l}
$$
\end{theorem}
\begin{proof}
The results easily follows from Theorem~\ref{theorem:sub_exponential_closed_cluster} applied to the construction presented in the proof of Theorem~\ref{theorem:subcritical_poisson_flight:threshold_bound}.
\end{proof}

\section{Super-critical Flight in a Poisson Forest} \label{section:supercritical_poisson_regime}

In this section, we continue our analysis of the single-integrator bird flying in a Poisson forest focusing on the super-critical regime using continuum percolation theory. In particular, we derive an upper bound on the critical speed for this special case and bound the scaling of the collision probability as a function of the time of flight.
That is, we try to answer the following question: {\em For how long can the bird survive without crashing into a tree, when it is flying with a speed that is larger than the critical threshold?}
We show that the survival probability of the bird is exponentially small in time of flight, independent of the planner that governs the bird. 

This section is organized as follows. First, we briefly introduce some important results in continuum percolation theory. Second, we describe a geometric model in which a certain type of percolation implies the non-existence of infinite collision-free trajectories. For the same case, we derive an upper bound on the critical speed, and show that the probability that a collision does not occur for $T$ time units is exponentially small in $T$ when the bird is flying at a speed above this bound.

\subsection{The Theory of Continuum Percolation} \label{section:poisson_regime:percolation_theory}

The canonical model in continuum percolation theory, called the {\em Gilbert disk model}, is a graph formed by connecting each point $\ElementPlane$ of the Poisson process $\PoissonProcess{\lambda}$ to all other points within a disc of diameter $D$ centered at $\ElementPlane$. 
This graph is said to {\em percolate} if it contains a connected component involving infinitely many vertices.
Surprisingly, the existence of such a ``large'' connected component exhibits a non-trivial phase transition. That is, there exists a finite non-zero critical intensity $\lambda_\mathrm{crit}$ such that (i) on one hand, for all $\lambda < \lambda_\mathrm{crit}$ the graph fails to percolate with probability one (ii) on the other hand, for all $\lambda > \lambda_\mathrm{crit}$ percolation does occur, in fact there exists a unique infinite component, with probability one (see~\cite{Meester:1996ue}).
In the rest of this section, we discuss this phenomenon and its implications in detail.
For reasons to become clear shortly, we will adopt the following view of the Gilbert disk model, rather than the graph-centric view described above.

Let $\OccupiedRegion{\lambda}{r}$ denote the region that is occupied by the disks with radius $r$ centered at the points of the Poisson process of intensity $\lambda$, i.e., $\OccupiedRegion{\lambda}{r} = \bigcup_{\ElementPlane \in \PoissonProcess{\lambda}} \ball{\ElementPlane}{r}$, where $\ball{\ElementPlane}{r}$ denotes the disk with radius $r$ centered at $\ElementPlane \in \plane$. The set $\OccupiedRegion{\lambda}{r}$ is also called the {\em occupied region}. An {\em occupied component} is a connected subset of $\OccupiedRegion{\lambda}{r}$ that is maximal with respect to set inclusion. 
Let $\OccupiedComponent{\lambda}{r}{0}$ denote the occupied component that intersects the origin.\footnote{Here, also, the choice of origin is irrelevant, since origin is probabilistically indistinguishable from any other point on the plane due to the stationarity of the Poisson process~\cite{Kingman:1993tg,Meester:1996ue} or any other ergodic spatial point process.} If no disc intersects the origin, then let $\OccupiedComponent{\lambda}{r}{0} = \emptyset$. Define the {\em critical intensity for occupied percolation}, denoted by $\lambda_c(r)$, as the smallest intensity such that the occupied component containing the origin is infinite with non-zero probability, i.e., 
$$
\lambda_c (r) := \inf \left\{ \lambda : \PP\left( \left\{\mu(\OccupiedComponent{\lambda}{r}{0}) = \infty\right\} \right) > 0 \right\}.
$$

It is rather easy to show that $\PP(\{\mu(\OccupiedComponent{\lambda}{r}{0}) = \infty\})$ is a non-decreasing function of $\lambda$ for all fixed $r$.
It can also be shown that $\PP(\{\mu(\OccupiedComponent{\lambda}{r}{0}) = \infty\}) = 0$ for small enough $\lambda$, and $\PP(\{\mu(\OccupiedComponent{\lambda}{r}{0}) = \infty\}) > 0$ for large enough $\lambda$~\cite{Meester:1996ue}.
Thus, for any $r > 0$ the critical intensity satisfies $\lambda_c(r) \in (0, \infty)$. In other words, there exists a non-trivial threshold, $\lambda_c(r)$, called the {\em percolation threshold}, below which the origin participates in an unbounded occupied component with probability zero, and above which the probability of the same event is non-zero. A Poisson process with intensity lower than the percolation threshold is said to be {\em sub-critical}, whereas one with intensity higher than the same threshold is  said to be {\em super-critical}.

An equivalent of Theorem~\ref{theorem:discrete_percolation} in the context of continuum percolation theory is stated below. This result establishes the existence of the aforementioned phase transition phenomenon.

\begin{theorem}[\cite{Meester:1996ue}] \label{theorem:percolation_occupied_region}
In the sub-critical regime, the probability that there exists an unbounded occupied component is zero. In the super-critical regime, there exists a unique unbounded occupied component with probability one.
\end{theorem}

The {\em vacant region} is defined as $\VacantRegion{\lambda}{r} := \plane \setminus \OccupiedRegion{\lambda}{r}$, i.e., the region that is {\em not} occupied by the disks.
A {\em vacant component} is a connected subset of $\VacantRegion{\lambda}{r}$ that is maximal with respect to set inclusion. 
Let $\VacantComponent{\lambda}{\Radius}{0}$ denote the vacant component that contains the origin. If the origin lies in an occupied component, then $\VacantComponent{\lambda}{r}{0} = \emptyset$. Define {\em critical intensity for vacant percolation} as
$$
\lambda^*_c (r) := \sup \left\{\lambda : \PP\left( \left\{ \mu(\VacantComponent{\lambda}{r}{0} = \infty) > 0 \right\} \right) \right\},
$$
i.e., the largest intensity for which the probability that the origin participates in an unbounded vacant component is non-zero. 
The following results play a central role in continuum percolation theory, establishing an important connection between the occupied and the vacant percolation. 
\begin{theorem}[\cite{Meester:1996ue}] \label{theorem:percolation_equality_of_critical_densities}
The occupied critical intensity and the vacant critical intensity are equal, i.e., $\lambda_c(r) = \lambda_c^*(r)$.
\end{theorem}

\begin{theorem}[\cite{Meester:1996ue}] \label{theorem:percolation_vacant_region}
In the super-critical regime, the probability that there exists an unbounded vacant component is zero. In the sub-critical regime, there is a unique unbounded vacant component with probability one.
\end{theorem}

Given a set $\SubsetPlane \subset \plane$, let $\mathrm{diam}(\SubsetPlane)$ denote the diameter of $\SubsetPlane$ defined as $\mathrm{diam}(\SubsetPlane) := \sup_{\ElementPlane_1, \ElementPlane_2 \in \SubsetPlane } \Vert \ElementPlane_1 - \ElementPlane_2\Vert $.
\begin{theorem}[\cite{Meester:1996ue}] \label{theorem:percolation_exponential_decay}
Let $\OccupiedComponent{\lambda}{\Radius}{\ElementPlane}$ and $\VacantComponent{\lambda}{\Radius}{\ElementPlane}$ denote occupied and vacant components that contains $\ElementPlane \in \plane$. Then, there exists constants $c_1,c_2,c_1',c_2' > 0$ such that in the sub-critical regime
$
\PP(\{\mathrm{diam}(\OccupiedComponent{\lambda}{\Radius}{\ElementPlane}) \ge \alpha\}) \le c_1 \, e^{-  c_2 \alpha},
$
and in the super-critical regime
$
\PP(\{\mathrm{diam}(\VacantComponent{\lambda}{\Radius}{\ElementPlane}) \ge \alpha\}) \le c_1' \, e^{-c_2' \alpha}.
$
\end{theorem}

Continuum percolation is known to depend on the intensity $\lambda$ of the Poisson process and the radius $r$ only through the expected number of points in a disk of radius $r$, i.e., $d:= \lambda \, \pi\, r^2$. Strictly speaking, this number is the expected degree of a vertex in the graph formed by the Gilbert disk model.
Define the {\em critical degree} as $\criticaldegree(r) :=  \lambda_c (r) \,\pi\, r^2$.
The exact value of the critical degree (hence, the percolation threshold) is not known. 
However, it was shown using rigorous analysis that the critical degree is between $2.184$ and $10.558$~\cite{Bollobas:2006ur}. Using Monte-Carlo integration techniques that offer rigorous bounds on the confidence interval, it was shown that the critical degree is between $4.508$ and $4.515$ with 99.99\% confidence~\cite{Balister:2005uj}. These results agree with the earlier Monte-Carlo simulation studies which suggest that the critical degree lies between $4.51218$ and $4.51228$~\cite{Quintanilla:2000uc}.

Another model that is of particular interest for the purposes of this paper is one in which a square of side length $l$, instead of a disk, is placed on each point of the Poisson process. Due to the obvious analogy, this model will be called the {\em Gilbert square model}, which has also been studied extensively in the literature.
In particular, Theorems~\ref{theorem:percolation_occupied_region}-\ref{theorem:percolation_exponential_decay} were shown to be valid for this model~\cite{Bollobas:2006ur}, 
and its critical degree was shown to be between $4.392$ and $4.398$ with 99.99\% confidence~\cite{Balister:2005uj}.

\subsection{The Primary Shadow Model}

In this section, we develop a random geometric model (much like the Gilbert model) in which percolation of a certain kind implies the non-existence of a reaching planner. 

Recall that $\ReachableStates{\speedratio}{z}{t}$ denotes set of all states that the dynamical system described by Equation~(\ref{eqn:system}) can reach before or on time $t$, starting from the state $z$. For the dynamics described by Equation~(\ref{eqn:sing_bird}), this set can be calculated easily. 
This cone-shaped region was depicted in Figure~\ref{figure:single_integrator_bird_reachable_set}.

Consider a forest composed of a single tree with radius $\Radius \in \reals_{>0}$ located at $\ElementPlane \in \plane$. Let $\LeftPrimaryShadow{\speed}{\ElementPlane}{\Radius}$ denote the set of all states starting from which the bird can not escape collision with this tree. The set $\LeftPrimaryShadow{\speed}{\ElementPlane}{\Radius}$ is illustrated in Figure~\ref{figure:left_shadow}.
Let $\RightPrimaryShadow{\speed}{\ElementPlane}{\Radius}$ denote the largest set of states that the bird can not reach if it starts at some state that is outside $\RightPrimaryShadow{\speed}{\ElementPlane}{\Radius}$. The set $\RightPrimaryShadow{\speed}{\ElementPlane}{\Radius}$, illustrated in Figure~\ref{figure:right_shadow}, is a mirror image of $\LeftPrimaryShadow{\speed}{\ElementPlane}{\Radius}$.
The sets $\LeftPrimaryShadow{\speed}{\ElementPlane}{\Radius}$ and $\RightPrimaryShadow{\speed}{\ElementPlane}{\Radius}$ are called the {\em left primary shadow region} and the {\em right primary shadow region}, respectively, of the tree with radius $\Radius$ located at $\ElementPlane$. Define also the set $\PrimaryShadow{\speed}{\ElementPlane}{\Radius} := \LeftPrimaryShadow{\speed}{\ElementPlane}{\Radius} \cup \RightPrimaryShadow{\speed}{\ElementPlane}{\Radius}$, which is called its {\em primary shadow region}. 

\begin{figure}[b]
\centering
\includegraphics[width = 0.6\linewidth]{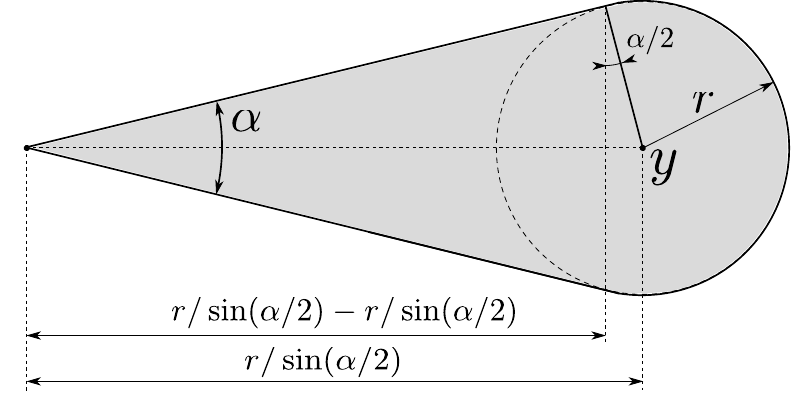}
\caption{The left primary shadow region $\LeftPrimaryShadow{\speedratio}{\ElementPlane}{\Radius}$ is shaded in grey. The figure also illustrates the length of the region in terms of $r$ and $\alpha$.}
\label{figure:left_shadow}
\end{figure}

\begin{figure}[b]
\centering
\includegraphics[width = 0.6\linewidth]{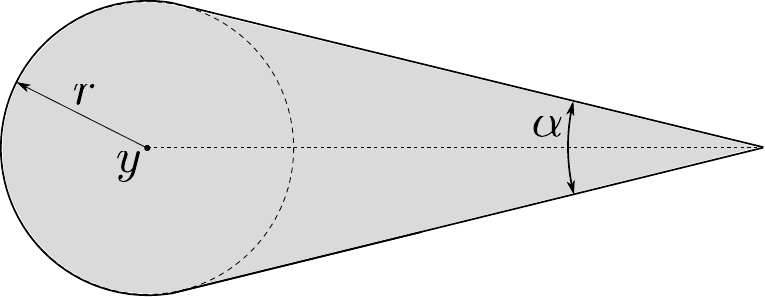}
\caption{The right primary shadow region $\RightPrimaryShadow{\speedratio}{z}{r}$ is shaded in grey. }
\label{figure:right_shadow}
\end{figure}

Consider a particular realization of the Poisson forest-generating process. Let $\{y_i\}_{i \in \naturals}$ denote the locations of the trees in this realization. Recall that each tree has radius $r$. The {\em occupied primary shadow region} is defined as $\bigcup_{i \in \naturals} \PrimaryShadow{\speed}{y_i}{r}$ and the {\em vacant primary shadow region} is defined as $\plane \setminus \bigcup_{i \in \naturals} \PrimaryShadow{\speed}{y_i}{r}$. The occupied/vacant primary shadow region is said to percolate if it contains an unbounded connected component. In the sequel, we refer to this percolation model as the {\em primary shadow model} with tree density $\rho$, tree radius $\Radius$, and speed $\speed$. The following lemma establishes an important connection between vacant percolation in the primary shadow model and the existence of infinite collision-free trajectories for the single-integrator bird flying in a Poisson forest.

\begin{lemma}
Suppose that the dynamics governing the bird is described by Equation~\eqref{eqn:sing_bird} and that the forest-generating process is Poisson. Then, a necessary condition for the existence an infinite collision-free trajectory under these assumptions is that the vacant primary shadow region percolates. 
\end{lemma}
\begin{proof}
Suppose that the vacant primary shadow region has no unbounded component. 
Distinguish between three cases: the bird starts its flight from (i) a left primary shadow region, (ii) a vacant primary shadow region, or (iii) a right primary shadow region. In all three cases, we show that any trajectory of the bird collides with a tree eventually. 

First, if the bird starts inside the left primary shadow region of some tree, then it must collide with the same tree eventually, merely by the definition of a left primary shadow region.

Second, if the bird starts at a vacant primary shadow region, then it must eventually enter either a left primary shadow region or a right primary shadow region, since the vacant primary shadow region has no unbounded component. However, it can not enter a right primary shadow region, since, by definition, right primary shadow regions are those that the trajectories of the bird can not reach from outside. So, it must eventually enter the left primary shadow region of some tree. But, then the bird must eventually collide with the same tree.

Finally, third, if the bird starts at a right primary shadow region, then it must eventually exit this region since the right primary shadow of any tree has bounded length. After this exit, by definition, the bird can not enter a new right primary shadow region. Hence, it must enter either into a vacant primary shadow region or it must enter a right primary shadow region. However, we have shown above that in either case, the bird must eventually collide with some tree.
\end{proof}

\subsection{The Non-existence of Infinite Collision-free Trajectories}

The lemma below provides a joint bound on $\rho$, $\Radius$, and $\speed$ such that when the model is operating above this bound vacant percolation fails to occur, almost surely, in the primary shadow model with tree density $\rho$, tree radius $\Radius$, and speed $\speed$.

\begin{lemma} \label{lemma:non_reaching_percolation_threshold}
Vacant percolation fails to occur, almost surely, in the primary shadow model with tree density $\rho$, tree radius $\Radius$, and speed $\speed$, whenever
$$
\frac{ \rho \, r^2}{\sin \alpha}   > \frac{\criticaldegree}{2},
$$
where $\alpha = 2 \, \tan^{-1} (1/\speed)$ and $\criticaldegree$ is the critical degree in the Gilbert square model with squares of unit side length.
\end{lemma}
First, note the following intermediate result. 
\begin{lemma} \label{lemma:scale_poisson_process}
Let $\poissonforest{\rho}$ be a homogeneous Poisson process on $\plane$ with intensity $\rho$. Let $l_1,l_2 > 0$, and define the function $f : \plane \to \plane$ as $f((z_1,z_2)) = (l_1\,z_1, l_2\,z_2)$. Then, $f(\poissonforest{\rho})$ is a homogeneous Poisson process with rate $\rho /( l_1 \, l_2)$.
\end{lemma}
\begin{proof}
The proof follows easily from the mapping theorem for Poisson processes noting that $f$ in this case is a continuous mapping with no atoms (see pp. 18 in~\cite{Kingman:1993tg}).
\end{proof}
\begin{proof}{\em (Proof of Lemma~\ref{lemma:non_reaching_percolation_threshold})}
Recall that $\speedratio$ is the speed in the dynamical system given by Equation~(\ref{eqn:system}).
Recall also that $\alpha := 2\, \tan^{-1}(1/\speed)$.
Given any $z \in \poissonforest{\rho}(\omega)$, notice that a rectangle with side lengths $r/\sin (\alpha/2)$ and $r/\cos (\alpha/2)$ fits into the primary shadow region $\PrimaryShadow{\speedratio}{\{z\}}{r}$ (see Figure~\ref{figure:rectangle_in_shadow_region}).

\begin{figure}[b]
\centering
\includegraphics[width=\linewidth]{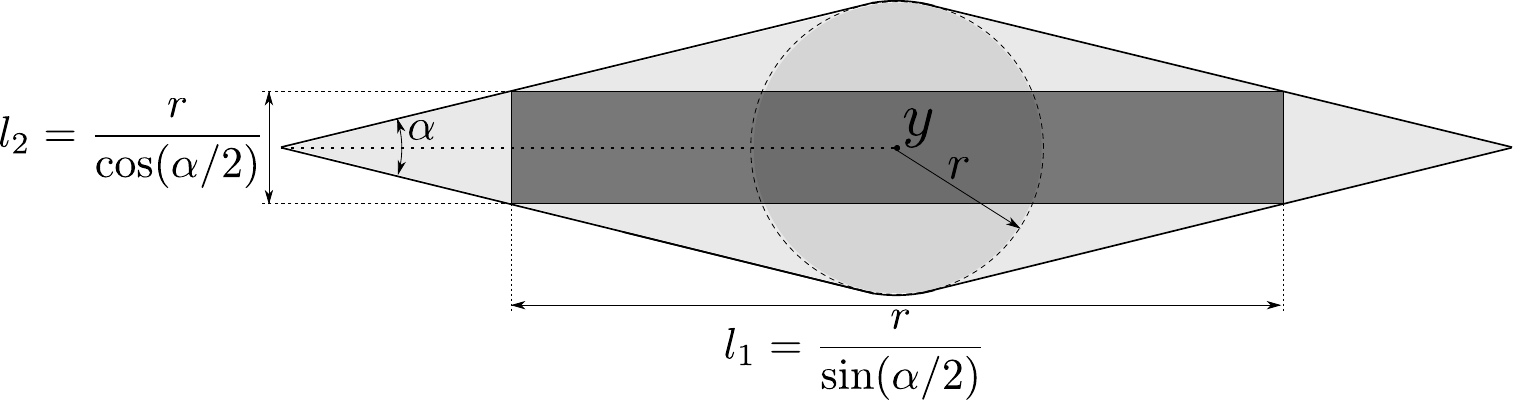}
\caption{A rectangle that can fit inside the primary shadow region $\PrimaryShadow{\speedratio}{\ElementPlane}{\Radius}$.}
\label{figure:rectangle_in_shadow_region}
\end{figure}

Clearly, the vacant percolation occurs in the primary shadow model with parameters $\rho$, $r$, and $\nu$ only if the vacant percolation occurs in the model where rectangles of side lengths $r/\sin (\alpha/2)$ and $r/\cos(\alpha/2)$ are attached to the each point of the Poisson process with intensity $\rho$. 

Consider a scaled process in which each point of the Poisson process $z = (z_1, z_2) \in \poissonforest{\rho}$ is mapped to $\big(\frac{z_1}{r/\sin(\alpha/2)}, \frac{z_2}{r/\cos (\alpha/2)}\big)$. By Lemma~\ref{lemma:scale_poisson_process}, the scaled process is also a homogeneous Poisson process with rate 
$
\frac{\rho \, r^2}{\cos(\alpha/2) \, \sin (\alpha/2)} = \frac{\rho \, r^2}{(1/2) \, \sin(\alpha)}. 
$
In the scaled process, each point of the Poisson process is associated with a unit square. Then, by Theorem~\ref{theorem:percolation_vacant_region} applied to the Gilbert square model, there is no unbounded vacant component in this scaled model, thus there is no unbounded component in the vacant shadow region, whenever 
$
\frac{\rho \, r^2}{\sin (\alpha)} > \frac{\criticaldegree}{2}.
$
\end{proof}

Hence, we state below the first main result of this section, which follows easily from Lemma~\ref{lemma:non_reaching_percolation_threshold}.
\begin{theorem} \label{theorem:poisson_regime:super_criticality}
Suppose that the dynamics governing the bird is described by Equation~\eqref{eqn:sing_bird} and the forest-generating process is Poisson with tree density $\rho$ and tree radius $\Radius$.
Then, there exists no infinite collision-free trajectory, almost surely, whenever the tree density $\rho$, tree radius $\Radius$, and speed $\speedratio$ satisfy the following:
$$
\frac{\rho \, r^2 }{ \sin \alpha} >  \frac{\criticaldegree}{2},
$$
where $\alpha = 2 \tan^{-1} (1/\speed)$ and $\criticaldegree$ is the critical degree in the Gilbert square model with squares of unit side length.
\end{theorem}

In the context of Problem~\ref{problem:main}, Theorem~\ref{theorem:poisson_regime:super_criticality} can be interpreted as follows. Whenever $\rho \, \Radius^2/ \sin(\Maneuverability) > \criticaldegree/2$, it is impossible to maintain flight indefinitely, with probability one, despite the knowledge of all the trees in the forest and regardless of the planner governing the motion of the bird. 

The following theorem strengthens Theorem~\ref{theorem:poisson_regime:super_criticality}, stating that the chances of maintaining flight with speed $\speed$ for some finite time $T$ converges to zero exponentially fast both with increasing $T$ and also with increasing $\speed$, again regardless of the planner governing the bird's motion, in the regime described in the statement of Theorem~\ref{theorem:poisson_regime:super_criticality}. 
\begin{theorem} \label{theorem:poisson_regime:exponential_decay}
Suppose that the conditions of Theorem~\ref{theorem:poisson_regime:super_criticality} are satisfied and that $\rho \, r^2 / \sin(\Maneuverability)> \criticaldegree/2$.
Let $\Event_\ElementPlane(\Time)$ denote the event that there exists a trajectory with speed $\speed$ that starts from $\ElementPlane$ and avoids collision with trees for the first $\Time$ time units.
Then, for any $\ElementPlane \in \plane$, there exists constants $c_1, c_2 > 0$ such that for all $\Time >0$
$$
\PP(\Event_{\ElementPlane}(\Time)) \le c_1 \, e^{-c_2 \, \Time}.
$$
\end{theorem}
\begin{proof}
Consider, again, the scaled process described in the proof of Lemma~\ref{lemma:non_reaching_percolation_threshold}. Given a set $\SubsetPlane \subseteq \plane$, recall that $\mathrm{diam}(\SubsetPlane)$ denotes the diameter of $\SubsetPlane$. Let $V'_{\ElementPlane}$ denote the vacant component that contains $\ElementPlane$, in the scaled process. By Theorem~\ref{theorem:percolation_exponential_decay}, there exists a constant $a'$ such that 
$$
\PP(\{\mathrm{diam}(W_\mathrm{max}') \ge \beta\}) \le e^{a' \beta}.
$$
Let $V_\ElementPlane$ denote the vacant region that contains $\ElementPlane$ in the primary shadow model. Let $\Length_\ElementPlane$ denote the length of this region in the longitudinal direction. Then, 
$$
\PP \left(\left\{ \Length_\ElementPlane \ge \beta \, \frac{r}{\sin (\alpha/2)}\right\}\right) \le e^{-a' \beta}.
$$
Define $\speed_0$ such that $\rho \, r^2 / \sin(\alpha(\speed)) > \criticaldegree/2$ for all $\speed > \speed_0$, where $\alpha(\speed) = 2\,\tan^{-1}(1/\speed)$.  Notice that for any such $\speed_0$, there exists some constant $\gamma > 0$, independent of $\speed$, such that $ \gamma \, \speedratio > 1/\sin (\alpha/2) = 1/\sin(\tan^{-1}(1/\speed))$, for all $\speed > \speed_0$. Then, for all $T > 0$ and all $\speed > \speed_0$,
\begin{align*}
\PP\big(\{ \Length_\ElementPlane \ge r\,\beta\,\gamma\,\speedratio\}\big) \;\; \le  \;\; \PP \left(\left\{ \Length_\ElementPlane \ge \beta \, \frac{r}{2 \, \sin \alpha} \right\}\right) \;\; \le \;\; e^{-a' \beta}.
\end{align*}
Finally, by a suitable change of variables, there exists a constant $a$, independent of $T$, such that 
$$
\PP \big(\{ \Length_\ElementPlane \ge T\,\speedratio\} \big) \;\; \le \;\; e^{-a \, T}.
$$
Since $T\,\speedratio$ is the distance (in longitudinal direction) that the bird can fly with speed $\speedratio$ in time $T$, the result follows.
\end{proof}

Finally, we are ready to provide a proof for Theorem~\ref{theorem:non_ergodic_counter_example}. 
\begin{proof}{\em (Proof of Theorem~\ref{theorem:non_ergodic_counter_example})} 
Take some $\rho$, $\Radius$, and $\speed$ such that $\rho \, \Radius^2 / \sin(\Maneuverability) < \log (1 / \sqrt{\CriticalProbabilitySite})$, where $\Maneuverability = 2\, \tan^{-1}(\speed)$ and $\CriticalProbabilitySite$ is the threshold for site percolation on the directed hexagonal lattice. 
Let $\rho'$ be large enough such that $\rho' \, \Radius^2 / \sin(\alpha) > \criticaldegree/2$, where $\criticaldegree$ is the critical degree for continuum percolation in the Gilbert square model with unit squares. 
Consider the mixed Poisson forest-generating process constructed as follows. The radii of the all trees are equal to $\Radius$. With probability $q$ the process generates the locations of the trees according to a homogeneous Poisson process with intensity $\rho$, and with probability $1-q$ it generates trees according to that with intensity $\rho'$. 
On one hand, when the trees are distributed according to the former process, there exists an infinite collision-free trajectory for the bird, with probability one, by Theorem~\ref{theorem:subcritical_poisson_flight:threshold_bound}. On the other hand, when the trees are distributed to the latter process, there exists no such trajectory, with probability one, by Theorem~\ref{theorem:poisson_regime:super_criticality}.
Then, the probability that there exists an infinite collision-free trajectory for the bird flying with speed $\speed$ in the forest generated by the mixed forest-generating process is $q$.
\end{proof}

\section{An Equivalent Geometric Model} \label{section:equivalent_model}

In the previous section, it was shown that the failure of vacant percolation in the primary shadow model implies the non-existence of infinite collision-free trajectories. In this section, we provide an improved model in which vacant percolation is {\em equivalent} to the existence of such trajectories. That is, vacant percolation occurs in the improved model if and only if there exists an infinite collision-free trajectory for the single-integrator bird through the Poisson forest.

Consider a forest with two trees that are located at $\ElementPlane_{1},\ElementPlane_{2} \in \plane$. Suppose both trees have radius $\Radius$. Recall that the set of all states starting from which collision with the tree located at $\ElementPlane_{i}$ is inevitable is denoted by $\LeftPrimaryShadow{\speed}{\ElementPlane_{i}}{\Radius}$, where $i \in \{1,2\}$. 
Let $\LeftShadow{\speed}{\{\ElementPlane_{1},\ElementPlane_{2}\}}{\Radius}$ denote the set of all states starting from which collision is inevitable with either the tree located at $\ElementPlane_{1}$ or that located at $\ElementPlane_{2}$. 
In general, the set $\LeftShadow{\speed}{\{\ElementPlane_{1}, \ElementPlane_{2}\}}{\Radius}$ is not the same as the set $\LeftPrimaryShadow{\speed}{\ElementPlane_{1}}{\Radius} \cup \LeftPrimaryShadow{\speed}{\ElementPlane_{2}}{\Radius}$. Consider for instance the case depicted in Figure~\ref{figure:counter_example}. If the single-integrator bird starts in the region that is shaded in dark grey, then collision with either one of the trees is inevitable, although this region is neither in $\LeftPrimaryShadow{\speed}{\ElementPlane_{1}}{\Radius}$ nor in $\LeftPrimaryShadow{\speed}{\ElementPlane_{2}}{\Radius}$.

\begin{figure}[b]
\centering
\includegraphics[width=0.7\linewidth]{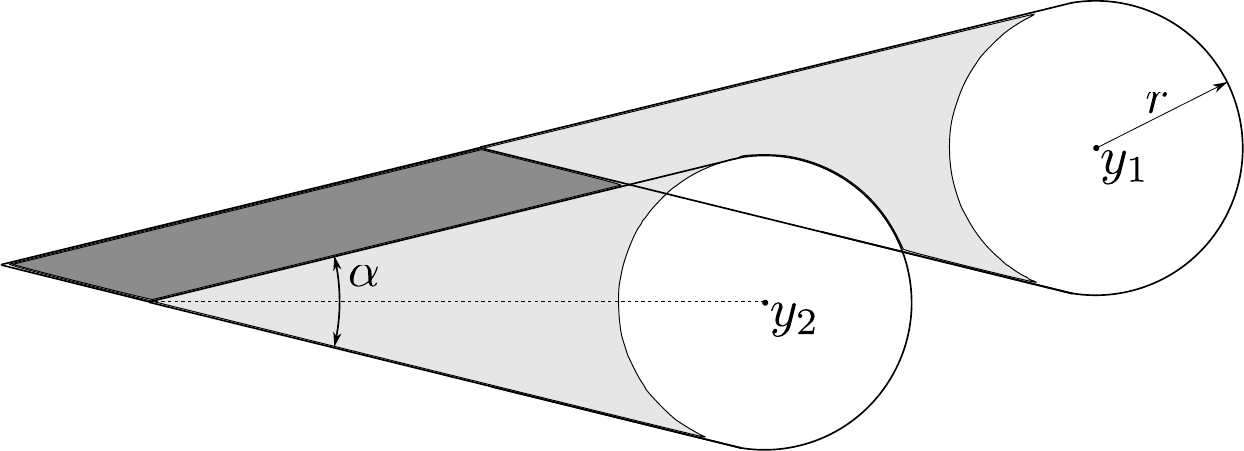}
\caption{The gray-shaded region is the set of all states starting from which collision with either the tree located at $y_1$ or that located at $y_2$ is inevitable. The dark grey region is the set of states that are neither in the primary left shadow region of the tree located at $y_1$ nor in that located at $y_2$.
}
\label{figure:counter_example}
\end{figure}

Given a left primary shadow region $\LeftPrimaryShadow{\speed}{\ElementPlane}{\Radius}$, define the {\em top boundary} and the {\em bottom boundary} of $\LeftPrimaryShadow{\speed}{\ElementPlane}{\Radius}$ as shown in Figure~\ref{figure:top_bottom_boundaries}. Given two left primary shadow regions, say $\LeftPrimaryShadow{\speed}{\ElementPlane_1}{\Radius}$ and $\LeftPrimaryShadow{\speed}{\ElementPlane_2}{\Radius}$ as above, define the {\em left induced shadow region}, denoted by $\InducedShadow{\LeftPrimaryShadow{\speed}{\ElementPlane_1}{\Radius}}{\LeftPrimaryShadow{\speed}{\ElementPlane_2}{\Radius}}$, as the set of all states starting from which the bird must go into either $\LeftPrimaryShadow{\speed}{\ElementPlane_1}{\Radius}$ or $\LeftPrimaryShadow{\speed}{\ElementPlane_2}{\Radius}$, thus eventually collide with either the tree located at $y_1$ or that located at $y_2$. See Figure~\ref{figure:induced_shadow_1}.
Notice that a left induced shadow region is formed only if the top boundary of one shadow intersects the bottom boundary of another shadow as depicted in the figure. If the opposite boundaries do not intersect, then the left induced shadow region is empty.

\begin{figure}[b]
\centering
\includegraphics[height=1.7cm]{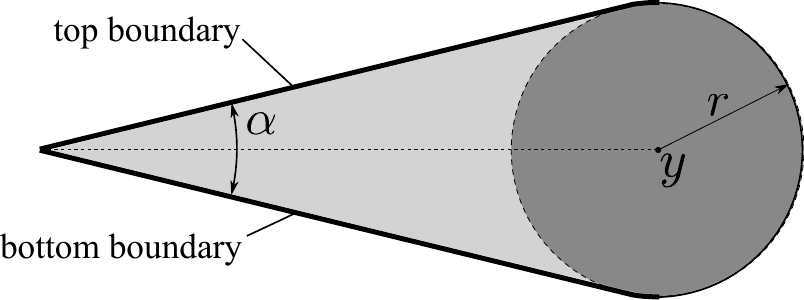}
\caption{The top and the bottom boundaries of a primary left shadow region.}
\label{figure:top_bottom_boundaries}
\end{figure}

\begin{figure}[b]
\centering
\includegraphics[width=0.85\linewidth]{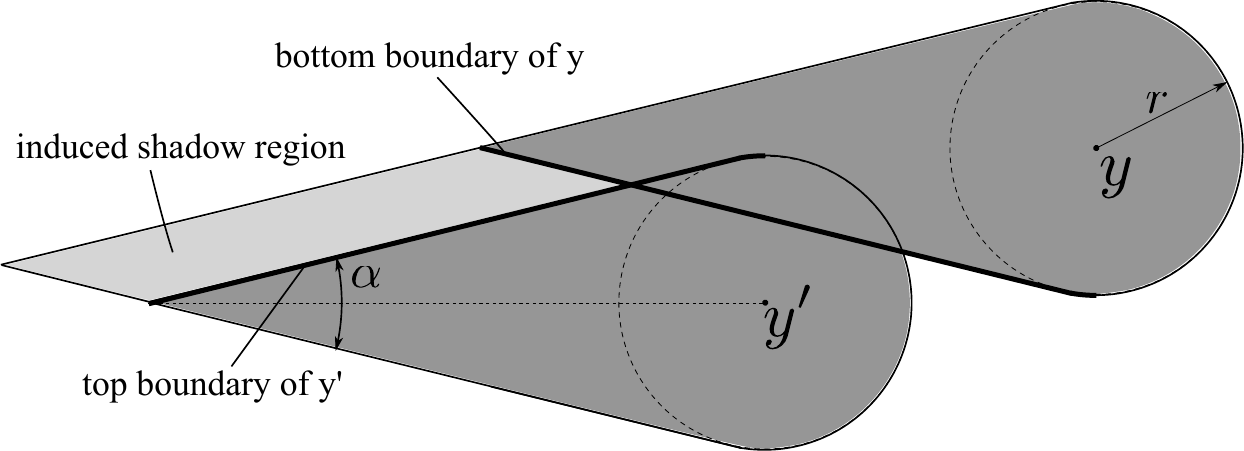}
\caption{The induced shadow region.}
\label{figure:induced_shadow_1}
\end{figure}

The top and bottom boundaries are defined for a left induced shadow region similarly as illustrated in Figure~\ref{figure:induced_shadow_boundaries}. In general, given two left shadows, say $\ShadowDummy_1$ and $\ShadowDummy_2$, either primary or induced, the set $\InducedShadow{\ShadowDummy_1}{\ShadowDummy_2}$ is the set of all states starting from which the bird must either enter $\ShadowDummy_1$ or $\ShadowDummy_2$. The set $\InducedShadow{\ShadowDummy_1}{\ShadowDummy_2}$ is nonempty whenever opposite boundaries of $\ShadowDummy_1$ and $\ShadowDummy_2$ intersect. See Figure~\ref{figure:induced_induced_shadow} for an illustration.

\begin{figure}[b]
\centering
\includegraphics[width=0.85\linewidth]{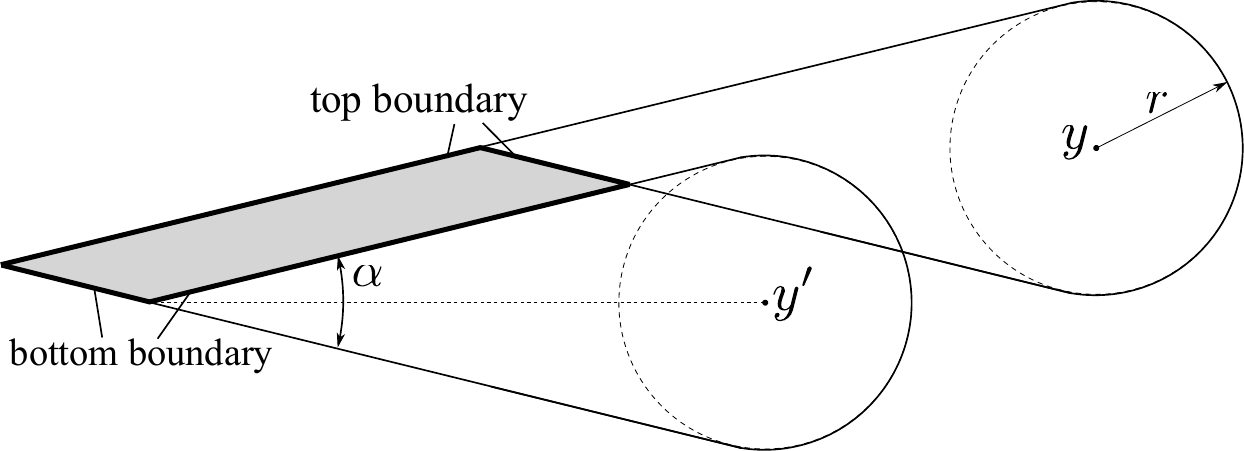}
\caption{The induced shadow resulting from the intersection of the bottom boundary of $\LeftPrimaryShadow{\speedratio}{\ElementPlane}{\Radius}$ with the top boundary of $\LeftPrimaryShadow{\speedratio}{\ElementPlane}{\Radius}$ is shown in light grey. The top and the bottom boundaries of the induced shadow region are shown in bold and labeled in the figure.}
\label{figure:induced_shadow_boundaries}
\end{figure}

\begin{figure}[b]
\begin{minipage}[h]{0.6\linewidth}
\begin{center}
\subfigure[]{
\centering
\includegraphics[height = 3cm]{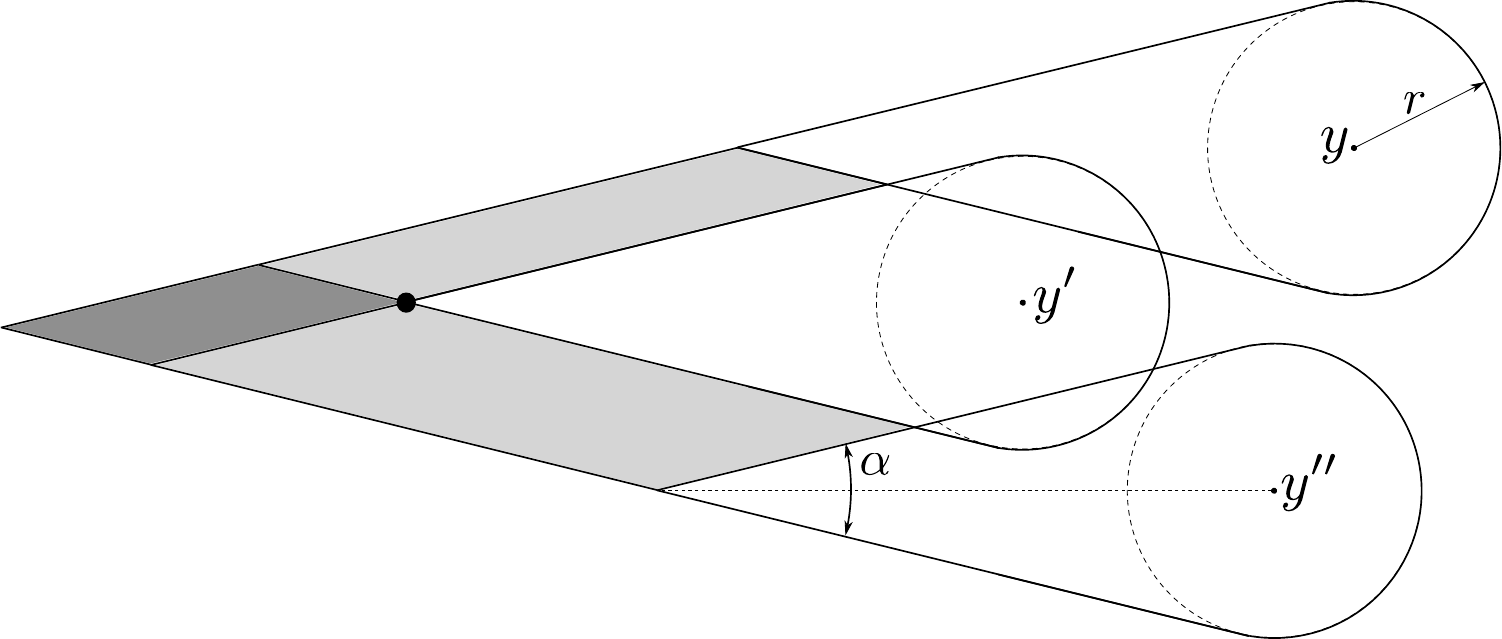}
}
\subfigure[]{
\centering
\includegraphics[height = 3cm]{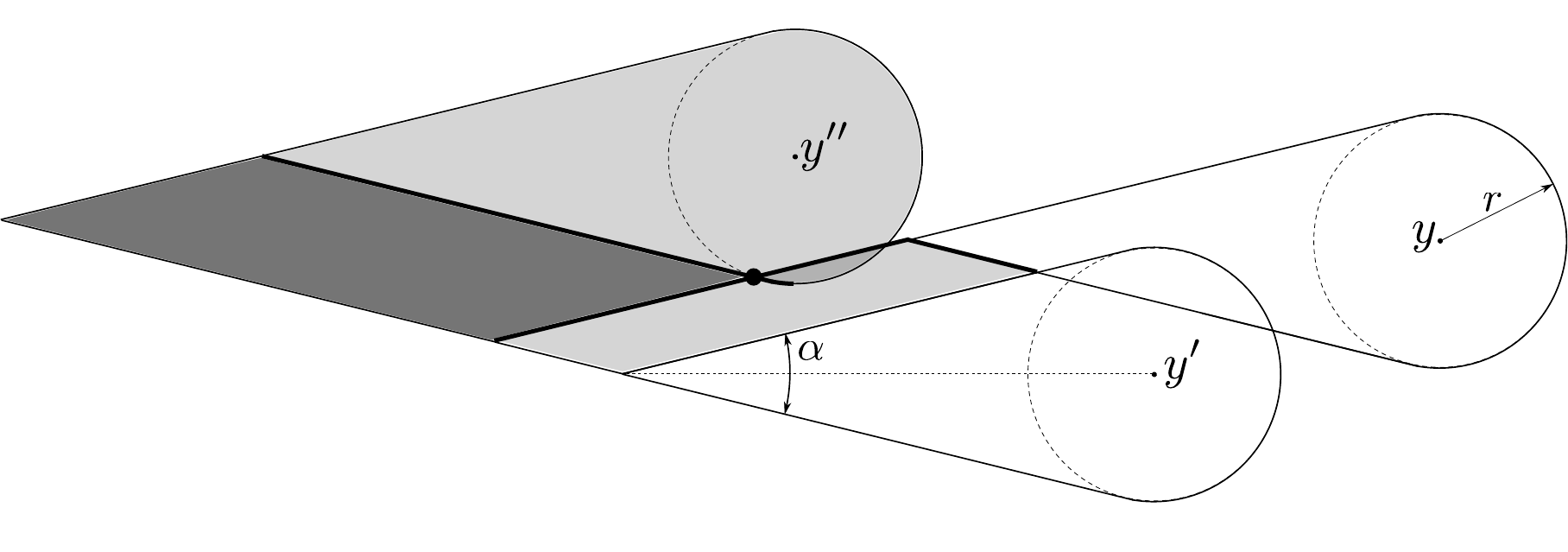}
}
\end{center}
\end{minipage}
\caption{In Figure (a), two induced shadows (shown in light grey) intersect to produce a new induced shadow region. Their intersection point is shown as a black dot. In Figure (b), an intersection of an induced region with a primary shadow region is shown. The intersection point of the top boundary of the induced shadow region and the bottom boundary of the primary shadow region is shown as a black dot. The resulting induced shadow is in dark grey.}
\label{figure:induced_induced_shadow}
\end{figure}

Given a finite or a countably infinite set $\SetTreeLocations \subset \plane$ of tree locations, let $\LeftShadow{\speed}{\SetTreeLocations}{\Radius}$ denote the set of all states starting from which collision with some tree in $\SetTreeLocations$ is inevitable. 
If $\SetTreeLocations$ is a finite set, then $\LeftShadow{\speed}{\SetTreeLocations}{\Radius}$ can be constructed in finite time by first constructing the induced shadow regions of all pairs, then constructing the induced shadow regions of all shadows, and so on, until there is not new induced shadow. This procedure is illustrated in Algorithm~\ref{algorithm:construct_shadows}. If $\SetTreeLocations$ is a countably infinite set, Algorithm~\ref{algorithm:construct_shadows} provides a countably enumerable procedure.

\IncMargin{0.05in}
\begin{algorithm}
$\mathrm{ActiveShadows} \leftarrow \bigcup_{\ElementPlane \in \SetTreeLocations} \{\LeftPrimaryShadow{\speedratio}{\ElementPlane}{\Radius}\}$\;
$\mathrm{AllShadows} \leftarrow \mathrm{ActiveShadows}$\;
\While{$\mathrm{ActiveShadows} \neq \emptyset$}{
	$\mathrm{NewShadows} \leftarrow \emptyset$\;
	\For{all $\ShadowDummy_1 \in \mathrm{ActiveShadows}$} {
		\For{all $\ShadowDummy_2 \in \mathrm{ActiveShadows}$} {
			$\mathrm{NewShadows} \leftarrow \mathrm{NewShadows}\cup \InducedShadow{\ShadowDummy_1}{\ShadowDummy_2}$\;
		}
	}
	$\mathrm{AllShadows} \leftarrow \mathrm{AllShadows} \cup \mathrm{NewShadows}$\;
	$\mathrm{ActiveShadows} \leftarrow \mathrm{NewShadows}$\; 
}
\Return{$\mathrm{AllShadows}$\;}
\caption{Geometric construction of $\LeftShadow{\speed}{\SetTreeLocations}{\Radius}$.}
\label{algorithm:construct_shadows}
\end{algorithm}

The set $\LeftShadow{\speed}{\SetTreeLocations}{\Radius}$ will be called the {\em occupied left shadow region}. The  {\em vacant left shadow region} is defined as $\plane \setminus \LeftShadow{\speed}{\SetTreeLocations}{\Radius}$. This model is called the {\em left shadow region model} with tree density $\rho$, tree radius $r$, and speed $\speed$, when the set $\SetTreeLocations$ is generated by a Poisson point process with intensity $\rho$. The notions of components, unbounded components, and percolation for both the occupied and the vacant regions are naturally extended to the left shadow region model. 
The left shadow region model is a novel percolation model that exhibits a number of interesting properties, some of which are unique even in the context of percolation theory. 

The first important property of the left shadow region model is its connection with the existence of infinite collision-free trajectories for the single-integrator bird. 
The following theorem is provided without a proof. Its proof can be carried out by an induction on the size of $\SetTreeLocations$. 

\begin{theorem} \label{theorem:infinite_shadows_and_left_shadow_percolation}
There exists an infinite collision-free trajectory of the single-integrator bird in a forest composed of trees with radius $\Radius$ located at points in $\SetTreeLocations$ if and only if the vacant left shadow region is non-empty.
\end{theorem}

An interesting corollary of this theorem is stated below.

\vspace{-0.1in}

\begin{corollary} \label{corollary:unbounded_left_shadow_region}
The vacant left shadow region is either empty or is unbounded.
\end{corollary}
\begin{proof}
Assume that the vacant left shadow region is non-empty. Let $\ElementPlane \in \plane$ be a point in the vacant left shadow region. Then, by Theorem~\ref{theorem:infinite_shadows_and_left_shadow_percolation}, there exists an infinite collision-free trajectory of the bird starting from $\ElementPlane$. Clearly, any point on this trajectory is in the vacant left shadow region also. Hence, the vacant left shadow region is unbounded, since the set of these points extend to infinity as the trajectory is infinite.
\end{proof}

Finally, the following theorem easily follows from the results presented in Theorems~\ref{theorem:subcritical_poisson_flight:threshold_bound}, \ref{theorem:poisson_regime:super_criticality}, and \ref{theorem:infinite_shadows_and_left_shadow_percolation} and Corollary~\ref{corollary:unbounded_left_shadow_region}.

\vspace{-0.1in}

\begin{theorem} \label{theorem:left_shadow_region_percolation}
Consider the left shadow region model with tree density $\rho$, tree radius $\Radius$, and speed $\speed$. For any $\rho$, $\Radius$, and $\speed$, there exists a finite threshold $\speed_\mathrm{crit}$, possibly zero, such that 
\begin{itemize}
\item for all $\speed > \speed_\mathrm{crit}$, there exists {\em no} vacant left shadow region, i.e., the vacant left shadow region is {\em empty}, almost surely, 
\item for all $\speed < \speed_\mathrm{crit}$, there exists an {\em unbounded} vacant left shadow region, almost surely.
\end{itemize}
Moreover, the critical speed $\speed_\mathrm{crit}$ is guaranteed to be non-zero whenever $\rho \, r^2 < 2 \log (1/\CriticalProbabilitySite)$ and to be equal to zero whenever $\rho \, r^2 > \criticaldegree/2$, where $\CriticalProbabilitySite$ is the critical probability for site percolation on the directed hexagonal lattice and $\criticaldegree$ is the critical degree for continuum percolation of unit squares.
\end{theorem}

Theorem~\ref{theorem:left_shadow_region_percolation} points out an interesting property of the left shadow region model. On one hand, in the sub-critical regime the model includes an unbounded vacant region. On the other hand, the vacant left shadow region is empty in the super-critical regime. Thus, when crossing the critical speed, which is guaranteed to be non-trivial when $\rho \, r^2$ is small enough, i.e., the forest is sparse enough, the once {\em unbounded} vacant shadow region suddenly completely disappears.
This property of the left shadow region model is unique even in the context of percolation theory, to the best of the authors' knowledge.

The authors conjecture, however, that there is some regime where the unbounded vacant left shadow region is not unique, which currently stands as an interesting open problem. We also conjecture that the uniqueness of the vacant left shadow region exhibits a phase transition of its own.

The occupied and the vacant right shadow region can be defined similar to the left shadow region. Clearly, the occupied or vacant percolation occurs in the right shadow model if and only if same kind of percolation occurs in the left shadow model. 
Recall that the primary shadow region for a particular tree was defined as the union of the left and right primary shadow regions of the same tree. 
Define the top boundary of a primary shadow of a particular tree as the union of the top boundaries of the left and right shadow regions for the same tree. The bottom boundary of a primary shadow is defined similarly. The induced shadow is defined as the union of the induced shadow regions for the left and right shadow regions, when two opposite boundaries of two distinct shadow regions intersect.
Define the {\em occupied shadow region} as the region generated by Algorithm~\ref{algorithm:construct_shadows} but with primary shadow regions (instead of primary left shadows). The {\em vacant shadow region} is defined as the part of the infinite plane that is not included in the occupied shadow region.
A standard argument using the stationarity of the forest-generating process shows that, with probability one, the occupied shadow region percolates if and only if the occupied left shadow region percolates.

\section{Computational Experiments} \label{section:computational}

In this section, the single-integrator bird flying in a Poisson forest is analyzed computationally. For this case, rigorous lower and upper bounds for the critical speed were derived in Sections~\ref{section:subcritical_poisson_regime} and Section~\ref{section:supercritical_poisson_regime}, respectively, and an equivalent percolation model was provided in Section~\ref{section:equivalent_model}. In this section, the same percolation model is used to approximately generate the phase diagram through Monte-Carlo simulations. 

The phase diagram is constructed in the $\speed$-$\rho$ plane, where $\speed$ is the speed of the single-integrator bird in the longitudinal axis and $\rho$ is the tree density of the Poisson forest.
Throughout the simulation study the tree radius is fixed to $\Radius = 1$. This assumption is without loss of any generality as noted in the proposition below. The same phase diagram applies to any $\Radius'$ by scaling the tree density $\rho$ with $\Radius'^2$.
\begin{proposition}
There exists a vacant region in the (primary) shadow model with tree density $\rho > 0$, tree radius $r > 0$, and speed $\speed > 0$ if and only if there exists a vacant region in the shadow model with tree density $\rho'$, tree radius $\Radius'$, and $\speed$ satisfying $\rho \, \Radius^2 = \rho' \, \Radius'^2$.
\end{proposition}
\begin{proof}
The results follows from the fact that the points of the Poisson process with intensity $\rho$ can be scaled by $\Radius'^2/\Radius^2$ to obtain the Poisson process with intensity $\rho'$. Since we obtain the same process in each cases, vacant region exists in the model driven with the first one if and only if it exists in that driven by the second one.
\end{proof}

In the computational experiments presented in this section, the existence of infinite collision-free trajectories through the Poisson forest is approximately verified by testing the existence of trajectories that traverse a large bounded region with width $\Width$ and length $\Length$.
Consider a a single-integrator bird flying at speed $\speed$ through a Poisson forests with tree density $\rho$ and tree radius $\Radius = 1$.
Then, roughly speaking, there exists a trajectory that traverses a region of width $\Width$ and length $\Length$, with arbitrarily high probability, whenever there exist an infinite collision-free trajectory through the Poisson forest, given that $\Width$ and $\Length$ are large enough. This statement can be made more precise using the result presented in Theorem~\ref{theorem:subcritical_poisson_regime:exponential_decay}.

Let $\Region{\Width}{\Length}$ denote the rectangular region with width $\Width$ and length $\Length$ centered at the origin.
To test the existence of a trajectory through $\Region{\Width}{\Length}$ in the Poisson forest, one can construct the shadow region generated by the trees in $\Region{\Width}{\Length}$. Then, clearly, there exists collision-free trajectory that traverses $\Region{\Width}{\Length}$ from left to right if and only if there exists a vacant shadow region in the shadow region model constructed using the trees in $\Region{\Width}{\Length}$. In fact, the latter occurs if and only if the largest-width occupied shadow component (in $\Region{\Width}{\Length}$) has width greater than or equal to $\Width$.

In the experiments, the width $\Width$ is set to 500 and the length $\Length$ is varied depending on the tree density $\rho$ such that the expected number of trees in $\Region{\Width}{\Length}$ is equal to 50,000, i.e., $\rho \, \Width \, \Length = 50,000$. In each experiment, the trees in $\Region{\Width}{\Length}$ is generated according to a Poisson process with intensity $\rho$, and the width of the largest-width occupied shadow component is noted after dividing by $\Width$. From now on, this quantity will be called the normalized maximum width. To obtain a statistical distribution of the normalized maximum width, the experiment is repeated 200 times for each of several tree density and speed pairs ($\rho$-$\speed$ pairs). 
The set of all such $\rho$-$\speed$ pairs is depicted in Figure~\ref{figure:computational_experiment:domain}.
To give the reader an idea, three realizations of the Poisson forest are shown together with the occupied shadow region for selected values of the tree density and speed in Figure~\ref{figure:computational_experiment:examples}.

\begin{figure}[b]
\begin{center}
\includegraphics[width=0.9\linewidth,trim=0in 2.35in 0in 2in, clip=true]{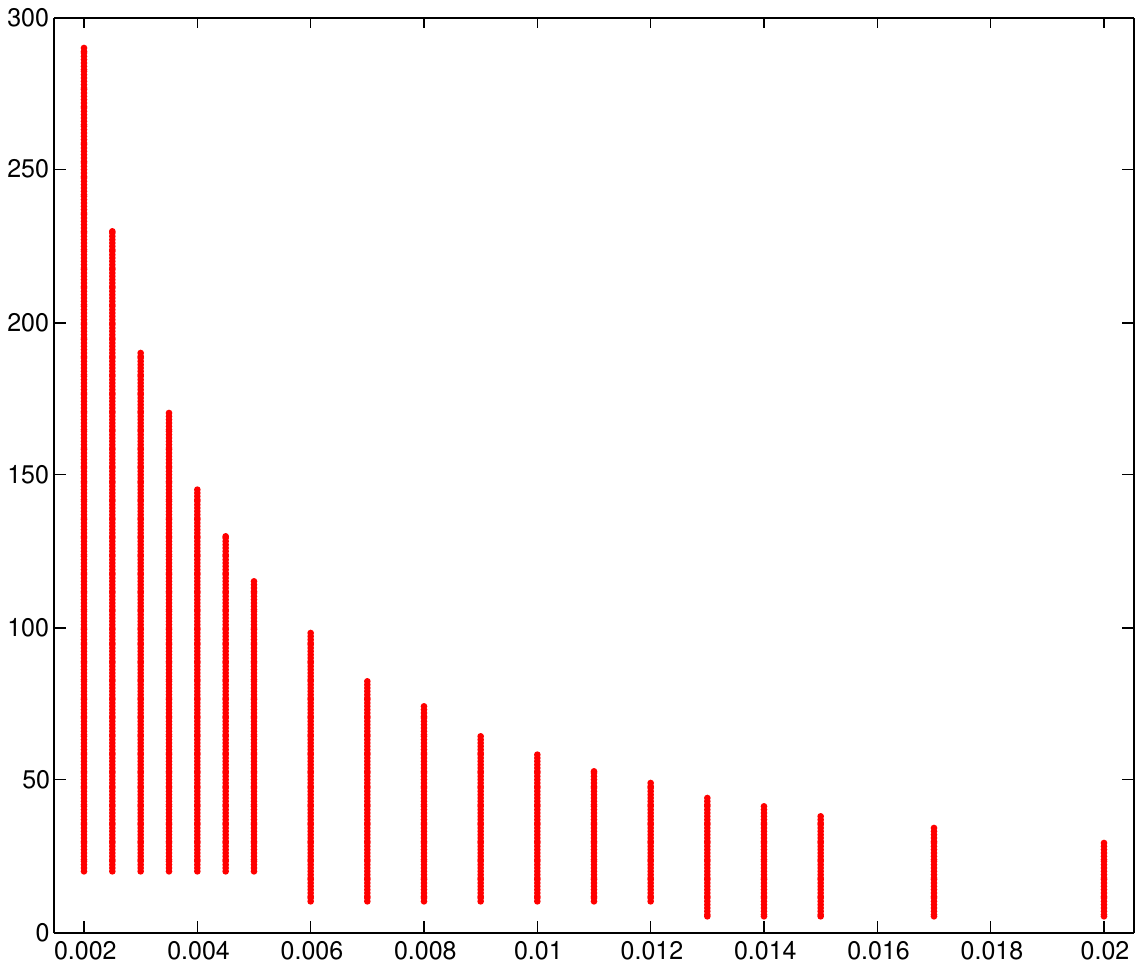}
\end{center}
\caption{The tree density (x-axis) and speed (y-axis) pairs that were considered in the computational experiment. }
\label{figure:computational_experiment:domain}
\end{figure}

\begin{figure*}
\begin{center}
\subfigure[]{\includegraphics[trim=0 4.45in 0 4.5in, clip=true, width=0.92\linewidth]{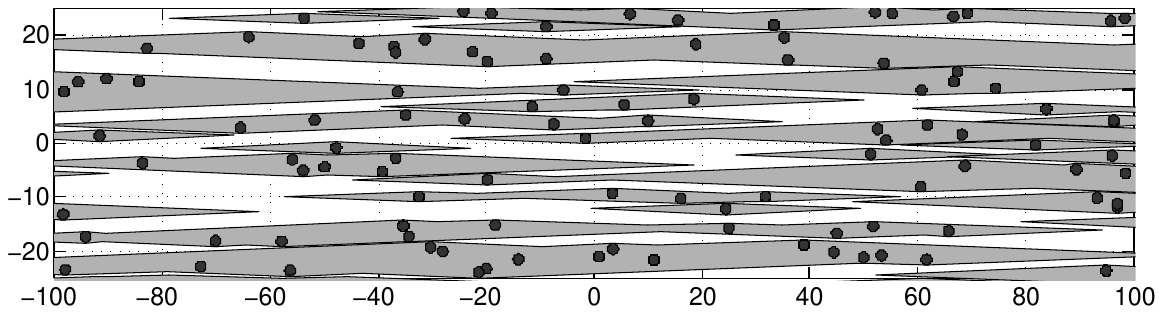}}
\subfigure[]{\includegraphics[trim=0 4.45in 0 4.5in, clip=true, width=0.92\linewidth]{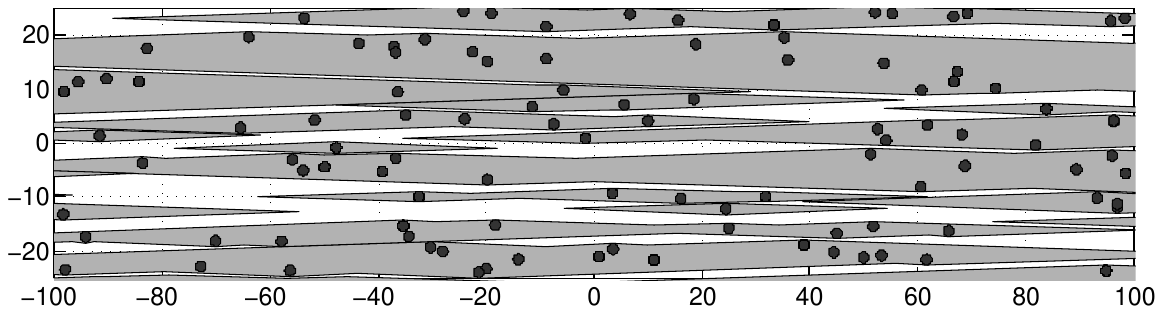}}
\subfigure[]{\includegraphics[trim=0 4.45in 0 4.5in, clip=true, width=0.92\linewidth]{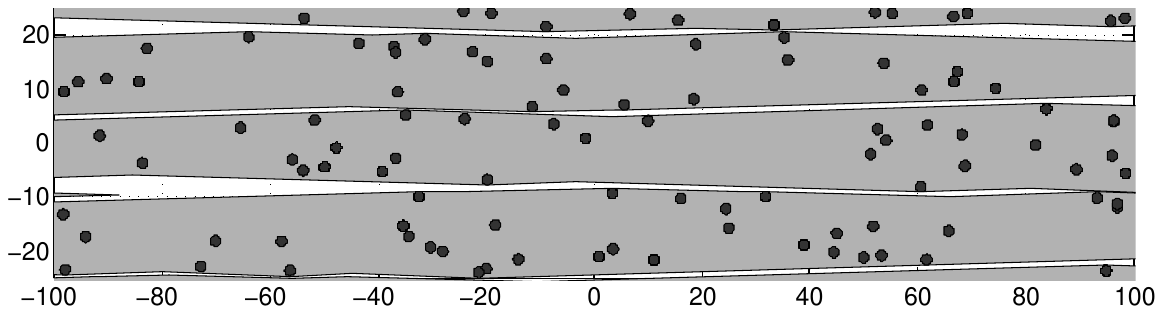}}
\end{center}
\caption{A 200 meter by 50 meter portion of a particular realization of the Poisson forest with a tree density of $0.01$ is shown. The trees are shaded in dark grey. In Figures (a), (b), and (c), the shadow regions are shown in light grey for velocities of $25$, $30$, and $35$, respectively. The white region is the set of all points starting from which flight with the corresponding speed can be maintained for at least till the end of the region that has a length of 10,000 meters.}
\label{figure:computational_experiment:examples}
\end{figure*}

In Figure~\ref{figure:computational_experiment:select_scatter}, the distribution of the normalized maximum width, obtained from 200 independent trials for each tree density and speed pair, is illustrated for three select values of tree density, namely $\rho = 0.003, 0.01, 0.017$, and several values of speed. 
In Figure~\ref{figure:computational_experiment:all_data}, normalized maximum width averaged over all trials is shown for each all of $\rho$-$\speed$ pairs that were considered in the experiments. The data is linearly interpolated for all other intermediate values.

It is clear from the figures that the maximum width changes rapidly, as predicted by the theory, near the critical speed. The scatter plots for different values of $\speed$ seem to be remarkably similar. In the authors' experience, similar distributions are obtained across different tree densities when the expected number of trees is kept constant. Moreover, if the region $\Region{\Width}{\Length}$ is scaled so that the expected number of trees falling in this region increases, the distributions of the maximum width seem to be more concentrated around their mean and the transition at the critical speed seems to be more rapid.

Finally, in Figure~\ref{figure:computational_experiment:bounds_and_data}, the lower and upper bounds as predicted by Theorems~\ref{theorem:subcritical_poisson_flight:threshold_bound} and \ref{theorem:poisson_regime:super_criticality} are given together with contour plots of the averaged normalized maximum width.

\begin{figure*}\centering
\subfigure[]{\includegraphics[width=0.3\linewidth, trim=0in 2.45in 0in 2in, clip=true]{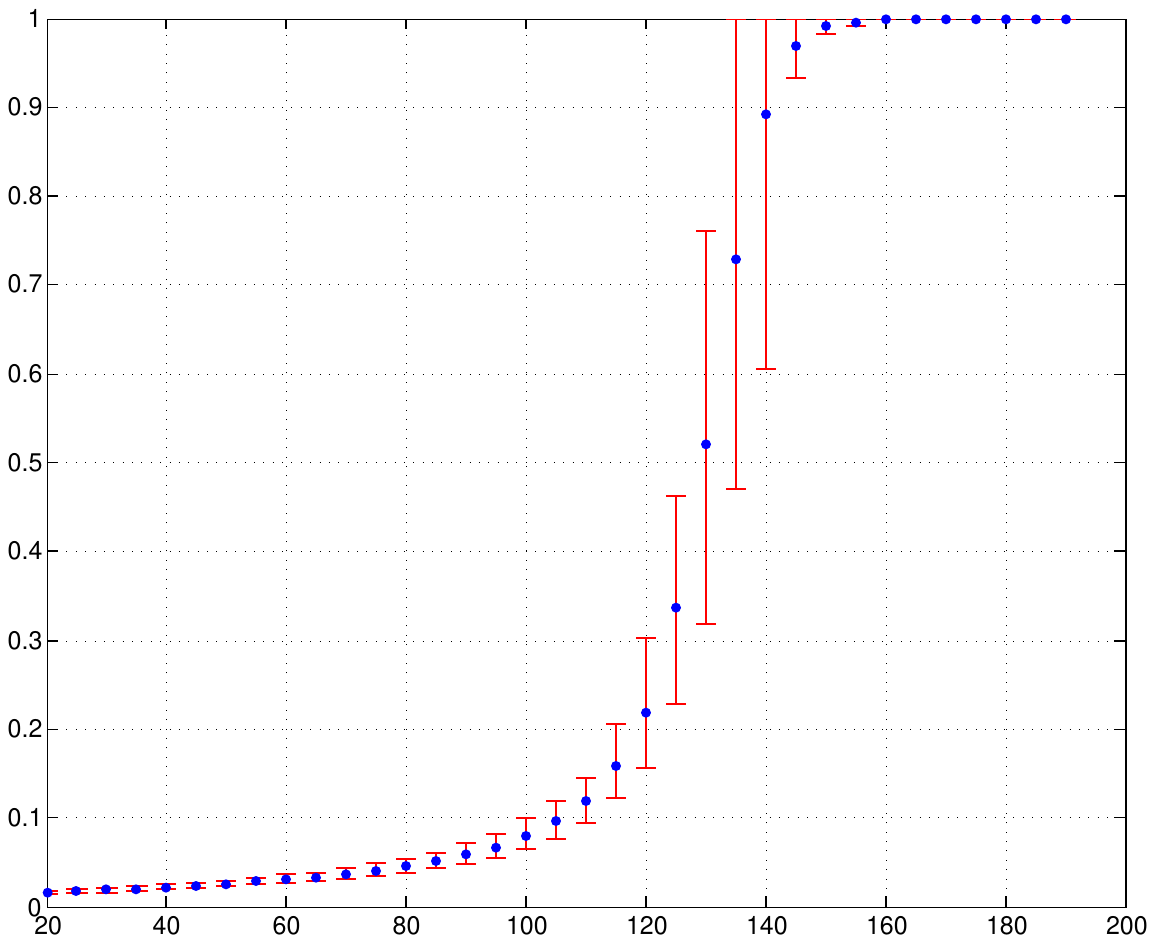}}
\subfigure[]{\includegraphics[width=0.3\linewidth, trim=0in 2.45in 0in 2in, clip=true]{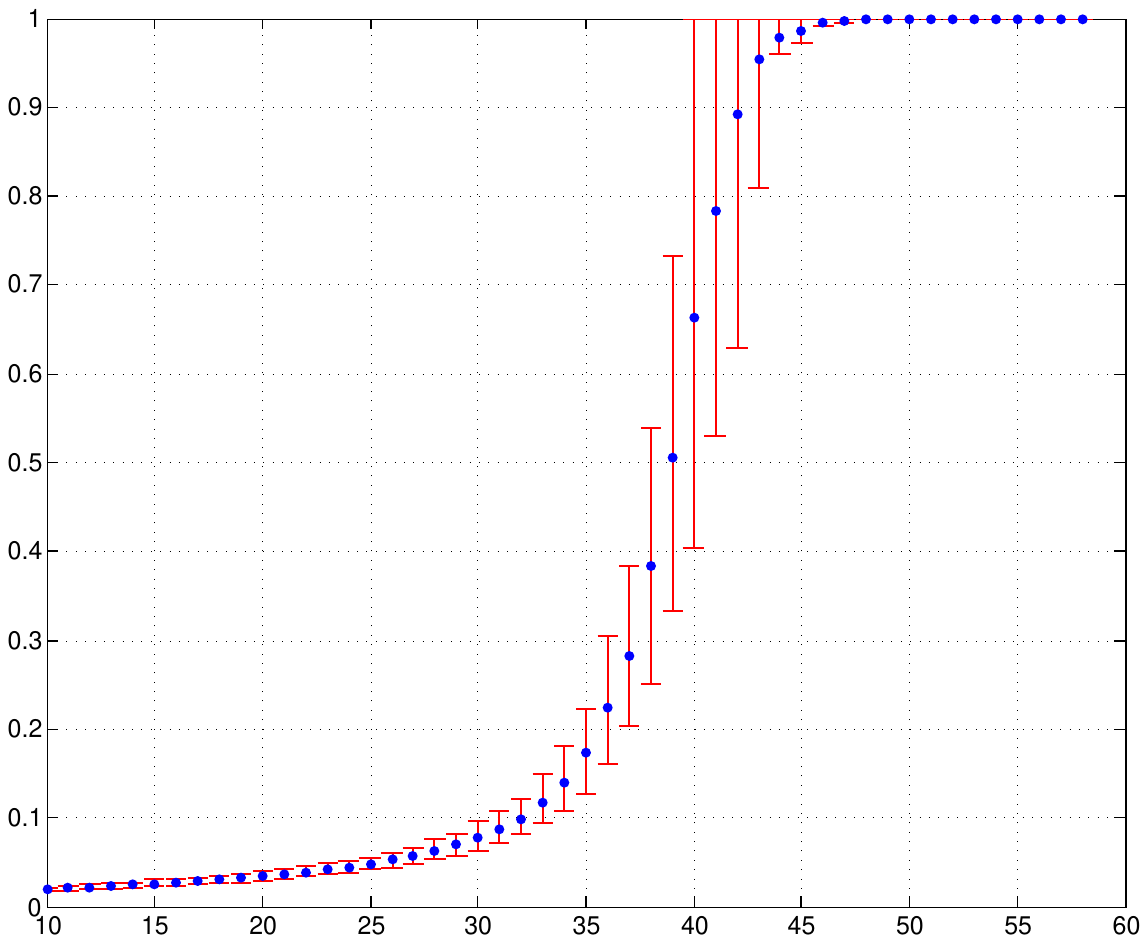}}
\subfigure[]{\includegraphics[width=0.3\linewidth, trim=0in 2.45in 0in 2in, clip=true]{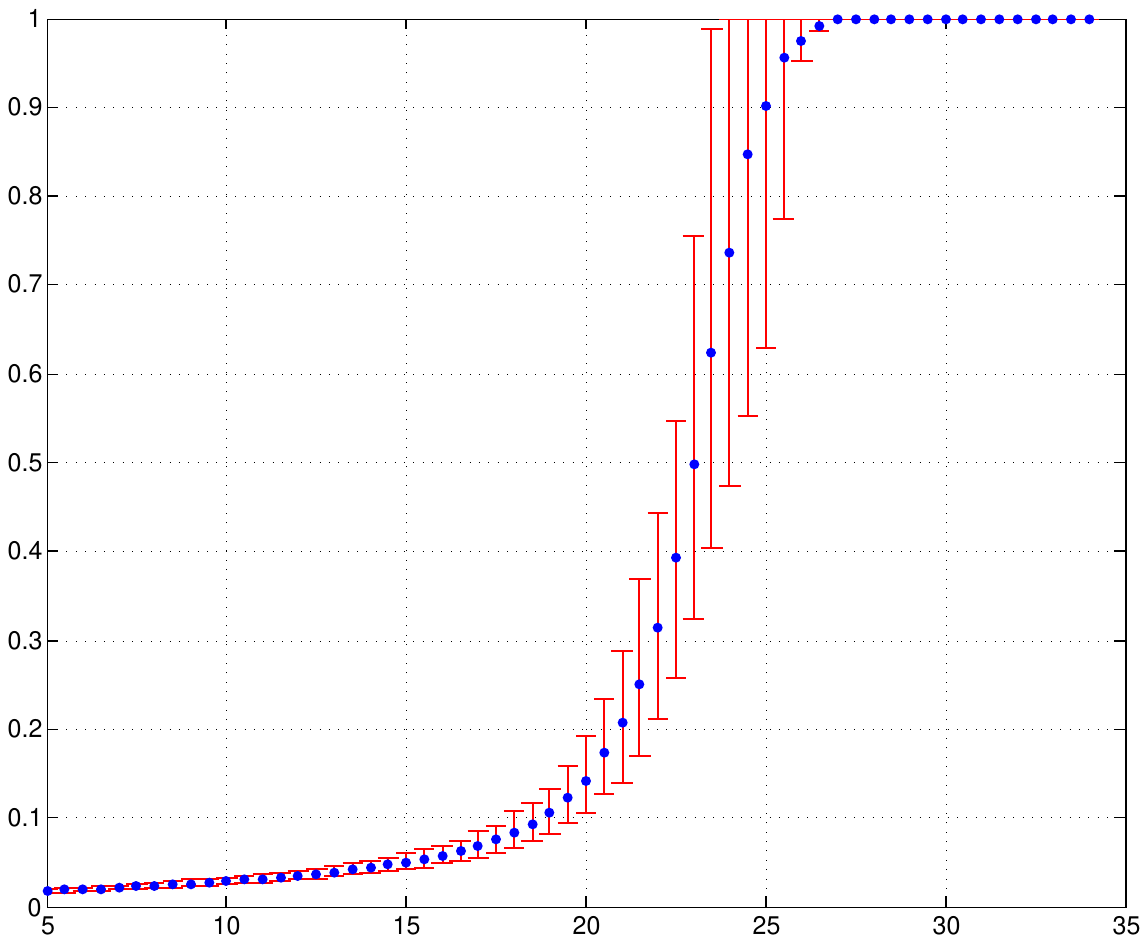}}\\
\subfigure[]{\includegraphics[width=0.3\linewidth, trim=0in 2.45in 0in 2in, clip=true]{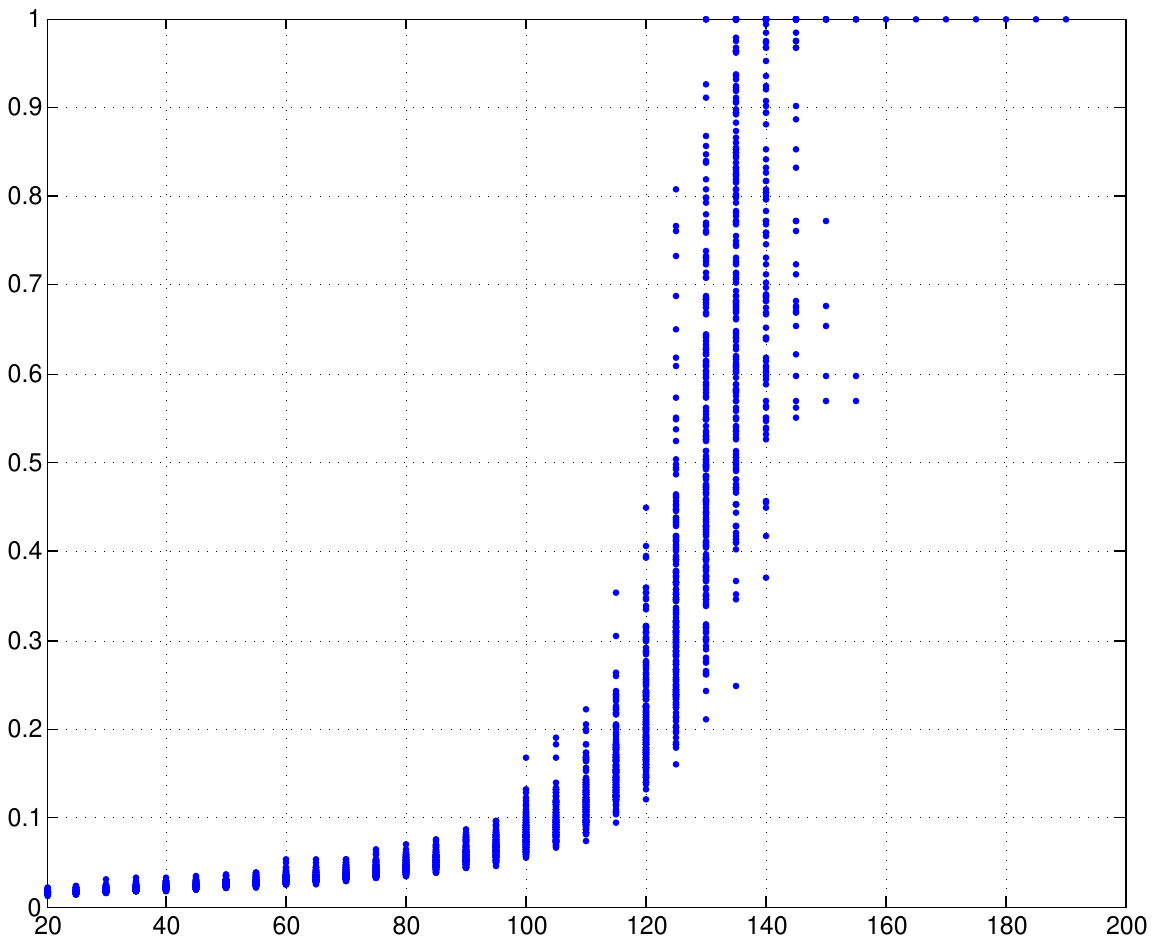}}
\subfigure[]{\includegraphics[width=0.3\linewidth, trim=0in 2.45in 0in 2in, clip=true]{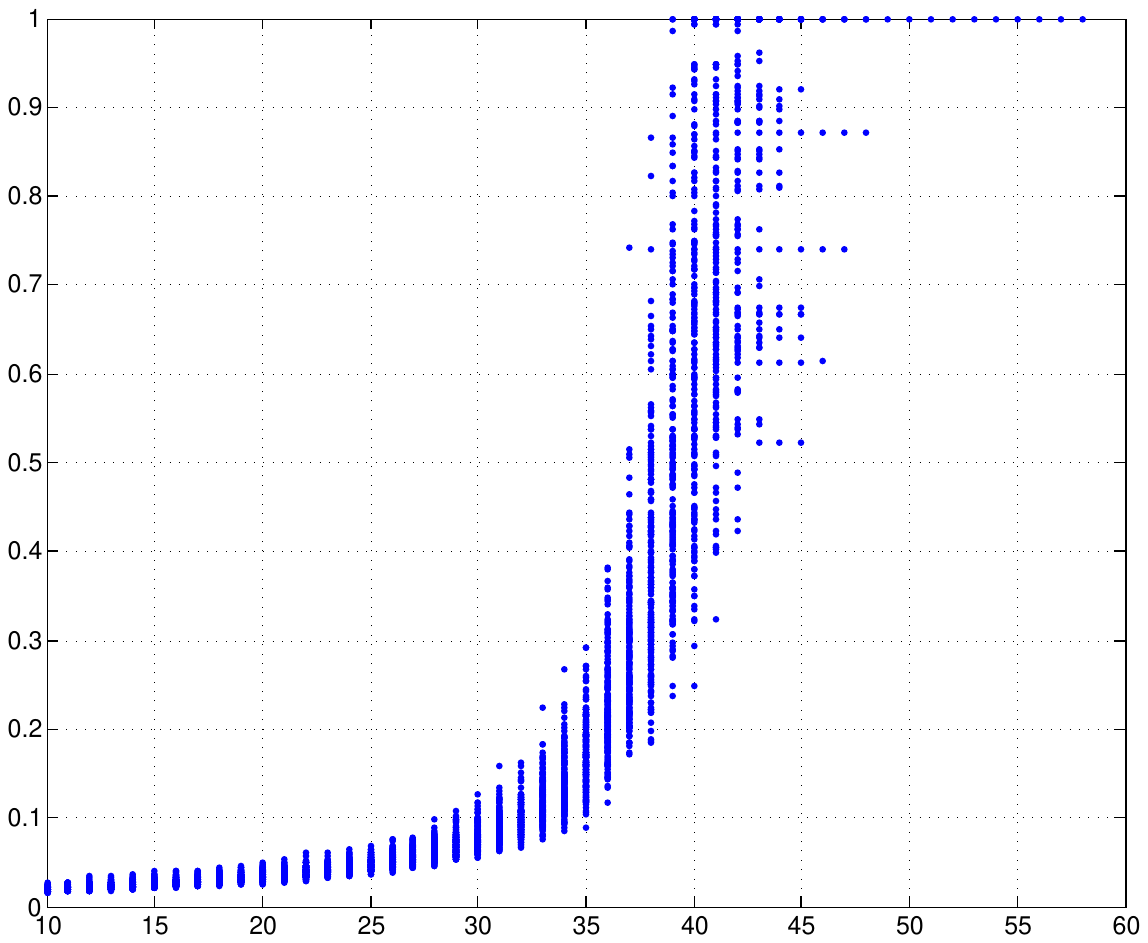}}
\subfigure[]{\includegraphics[width=0.3\linewidth, trim=0in 2.45in 0in 2in, clip=true]{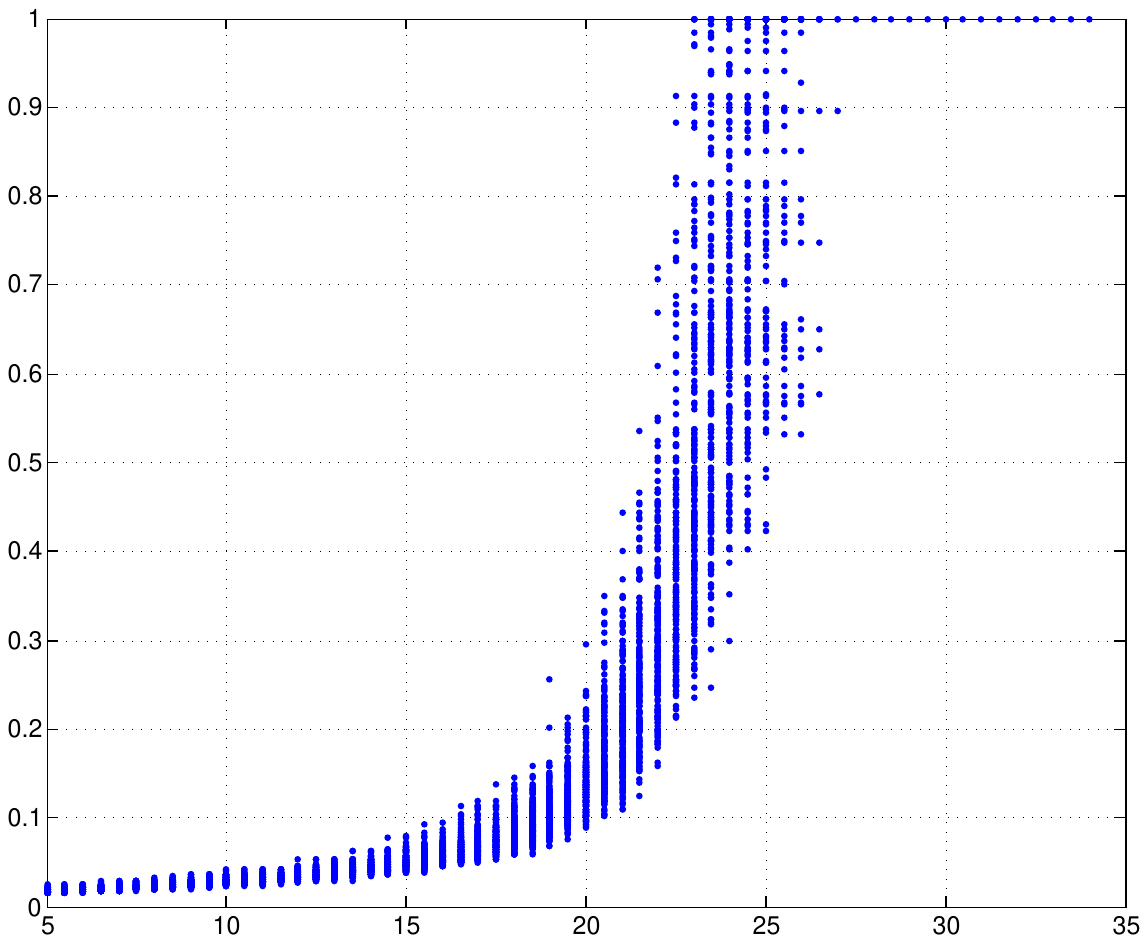}}
\caption{
Normalized maximum width versus the speed is shown for tree densities of 0.003, 0.010, and 0.017 in plots (a), (b), and (c), respectively. Average normalized maximum width values are shown in blue and the 10\% and 90\% percentiles of its distribution are shown by red bars.
The distribution of the normalized maximum width are shown in scatter plots for the same tree densities in Figures (d), (e), and (f), respectively, for various values of speed.}
\label{figure:computational_experiment:select_scatter}
\end{figure*}

\begin{figure*}
\begin{center}
\includegraphics[trim = 0 0.2in 0 0, clip=true, width=0.05\linewidth]{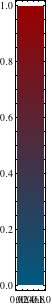}\qquad
\subfigure[]{\includegraphics[width=0.488\linewidth]{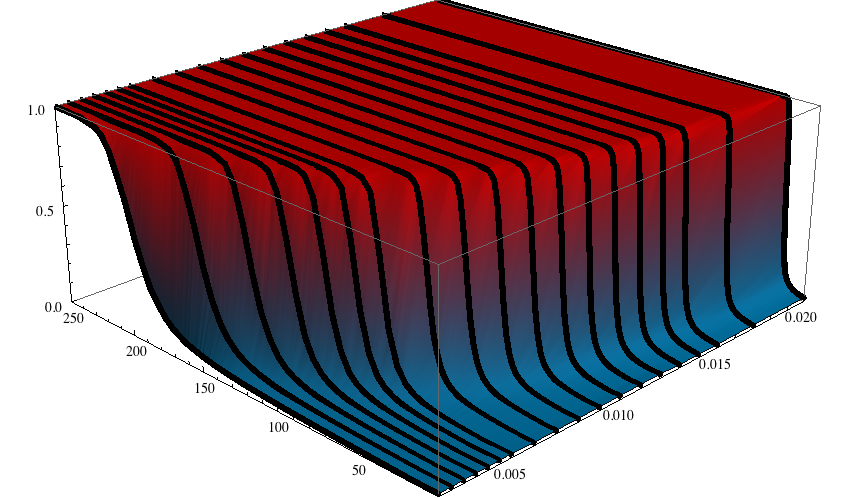}}
\subfigure[]{\includegraphics[width=0.294\linewidth]{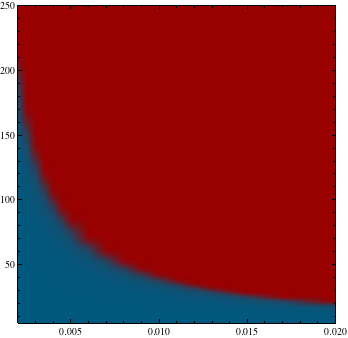}}
\end{center}
\caption{Normalized maximum width is shown for several values of the tree density and speed pairs in Figure (a). The black lines are linearly interpolated and extrapolated values obtained through computational experiments. The surface is a linear interpolation of the data represented by the black lines. The surface is color coded according to normalized maximum width. A legend for the color coding is given on the left. The top view of the surface presented in Figure (a) is shown in Figure (b), where the rapid transition from small values of the normalized maximum width to large values can be observed clearly.}
\label{figure:computational_experiment:all_data}
\end{figure*}

\begin{figure}
\centerline{\includegraphics[width=0.9\linewidth]{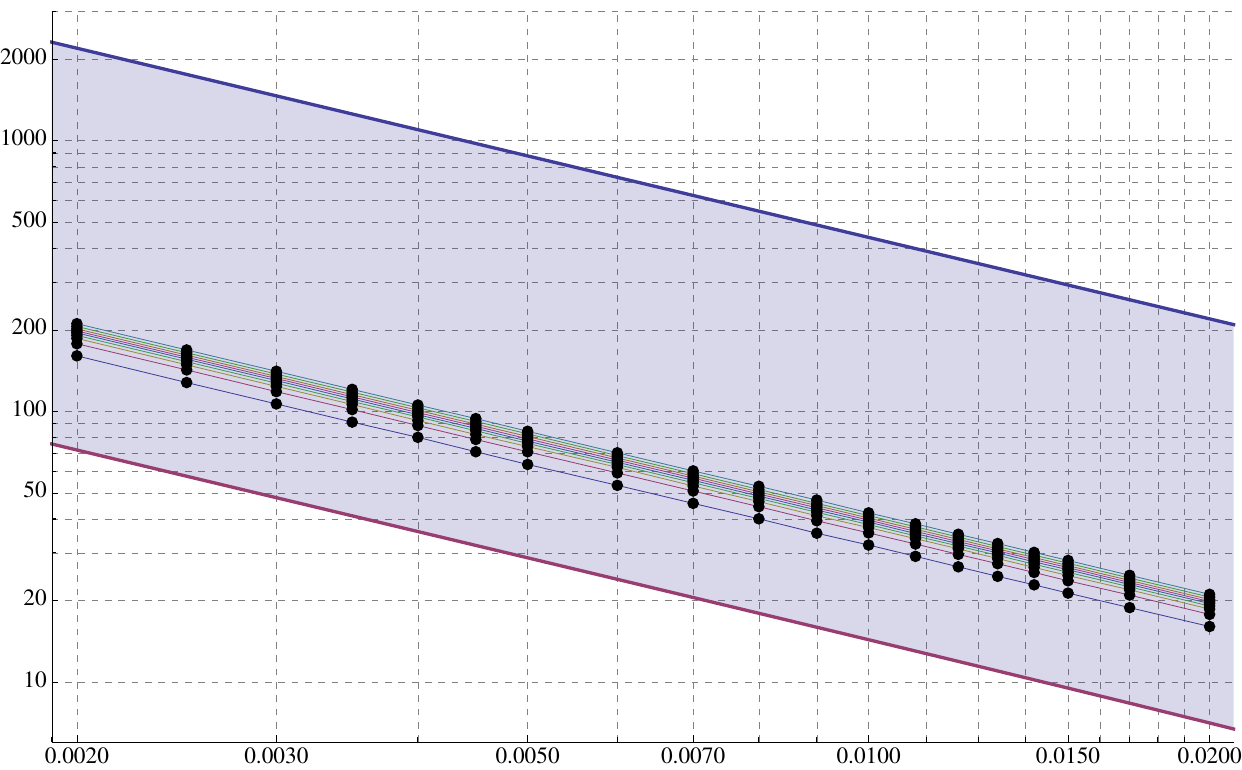}}
\caption{Rigorous upper and lower bounds for the critical speed are shown together with the level sets of normalized maximum width obtained using computational experiments. Black dots represent the tree density and speed pairs corresponding to the averaged maximum width at several values ranging from 0.1 to 0.9 with increments of 0.1. The axes are in logarithmic scale showing tree density (x axis) versus the velocity (y axis).}
\label{figure:computational_experiment:bounds_and_data}
\end{figure}

\section{Conclusion} \label{section:conclusion}

This paper considered a novel class of 
motion planning problems in stochastically-generated environments, where the statistics of the obstacle generation process is known, but the location and the shape of the obstacles are not known {\em a priori}. 
The problem was motivated by a bird flying through a dense forest environment. As a model, a planar environment forest environment in which the trees are randomly-placed disks with a random radius, was proposed, and a dynamics of the bird was parametrized by a speed parameter. 

In the case when the stochastic point process generating the location and sizes of the trees is ergodic, it was shown that the existence of infinite trajectories, i.e., those trajectories that diverge towards infinity, that are collision free is tied to a novel phase transition result: there exists a critical speed such that (i) when the bird is flying just above the critical speed, there exists no infinite collision-free trajectory, with probability one, (ii) when the bird is flying just below the critical speed, there exists at least one infinite collision-free trajectory, with probability one. 
The special case in which the bird is governed by a simple single-integrator dynamics and the locations of the trees is generated by a Poisson process and the radii of the trees are the same is also considered. Lower and upper bounds on the critical speed for this case are derived using discrete and continuum percolation theory, respectively. 
Moreover, for the same case, an equivalent percolation model is derived. Using this model, the phase diagram of high-speed flight is computed approximately through Monte-Carlo simulations. 

There are many directions for future work. In particular, we will consider the case when the bird has only limited range, i.e., not provided with the information of all the trees in the forest, and search for positive results, e.g., the probability that a bird with limited range can fly a particular distance in the sub-critical regime. 
The theory presented in this paper can be used to analyze a wide range of planning problems in which obstacles or other agents are encountered according to a spatio-temporal stochastic process. Our future work will also include identifying these problems that have practical applications.

\bibliography{karaman.frazzoli.tro11}
\bibliographystyle{unsrt}

\end{document}